\def\final{1}
\def\short{0}
\def\submission{0}

\ifnum\short=0
\newcommand{\longshort}[2]{#1}
\else
\newcommand{\longshort}[2]{#2}
\fi

\ifnum\final=0 
\documentclass[10pt]{article}
\else
\documentclass[11pt]{article}
\fi

\usepackage{algorithm, algorithmic}
\usepackage{amsmath, amssymb, amsthm}
\usepackage[inline]{enumitem}
\usepackage{framed}
\usepackage{verbatim}
\usepackage{color}
\usepackage{microtype}
\ifnum\short=0 
	\ifnum\submission=0
                \usepackage{kpfonts}
                \DeclareMathAlphabet{\mathsf}{OT1}{cmss}{m}{n}
                \SetMathAlphabet{\mathsf}{bold}{OT1}{cmss}{bx}{n}
                
                \usepackage[margin=1.01in]{geometry}
	\else
		\usepackage{times}
		\usepackage[margin=1.0in]{geometry}
	\fi
\else
	\usepackage{times}
	\usepackage[margin=1.0in]{geometry}
\fi 
\usepackage{multirow, tabularx}
\definecolor{DarkGreen}{rgb}{0.2,0.6,0.2}
\definecolor{DarkRed}{rgb}{0.6,0.2,0.2}
\definecolor{DarkBlue}{rgb}{0.15,0.15,0.55}
\definecolor{DarkPurple}{rgb}{0.4,0.2,0.4}

\usepackage[pdftex]{hyperref}
\hypersetup{
    linktocpage=true,
    colorlinks=true,				% false: boxed links; true: colored links
    linkcolor=DarkBlue,		% color of internal links
    citecolor=DarkBlue,	% color of links to bibliography
    urlcolor=DarkBlue,		% color of external links
}

\ifnum\final = 0
\newcommand{\mynote}[1]{\marginpar{\tiny #1}}
\newcommand{\Bignote}[1]{{\tiny #1}}
\else
\newcommand{\mynote}[1]{}
\newcommand{\Bignote}[1]{}
\fi

\newcommand{\rcomm}[1]{\Bignote{\color{blue} Raef: {#1}}}

\newcommand{\INDSTATE}[1][1]{\STATE\hspace{#1\algorithmicindent}}

\newcolumntype{Y}{>{\centering\arraybackslash}X}

\newcommand{\pr}[2]{\underset{#1}{\mathbb{P}}\left[ #2 \right]}
\newcommand{\ex}[2]{\underset{#1}{\mathbb{E}}\left[ #2 \right]}

\newcommand{\paren}[1]{{\left( {#1}\right)}}
\newcommand{\poly}{\mathrm{poly}}

\newcommand{\zo}{\{0,1\}}

\newcommand{\getsr}{\gets_{\mbox{\tiny R}}}

\newcommand{\from}{:}

\newcommand{\eps}{\varepsilon}
\DeclareMathOperator*{\argmin}{arg\,min}
\DeclareMathOperator*{\argmax}{arg\,max}
\newcommand{\eqand}{\qquad \textrm{and} \qquad}

\newcommand{\N}{\mathbb{N}}
\newcommand{\R}{\mathbb{R}}
\newcommand{\cA}{\mathcal{A}}
\newcommand{\cB}{\mathcal{B}}

\newcommand{\cM}{\mathcal{M}}

\newcommand{\cR}{\mathcal{R}}
\newcommand{\cU}{\mathcal{U}}
\newcommand{\cW}{\mathcal{W}}

\newtheorem{theorem}{Theorem}[section]
\newtheorem{thm}[theorem]{Theorem}
\newtheorem{lemma}[theorem]{Lemma}
\newtheorem{lem}[theorem]{Lemma}

\newtheorem{claim}[theorem]{Claim}

\newtheorem{corollary}[theorem]{Corollary}

\theoremstyle{definition}

\newtheorem{definition}[theorem]{Definition}

\newcommand{\tvd}{\mathrm{d_{TV}}}
\newcommand{\univ}{\mathcal{X}}
\newcommand{\dist}{\mathbf{P}}
\newcommand{\samp}{x}
\newcommand{\sample}{\mathbf{\samp}}

\newcommand{\query}{q}
\newcommand{\mechanism}{\cM}
\newcommand{\adv}{\cA}
\newcommand{\monitor}{\cW}
\newcommand{\alg}{\cW}
\newcommand{\queries}{k}
\newcommand{\queryset}{Q}
\newcommand{\querysetSQ}{\queryset_{\mathit{SQ}}}
\newcommand{\querysetDelta}{\queryset_{\Delta}}
\newcommand{\querysetCM}{\queryset_{\mathit{CM}}}
\newcommand{\querysetMIN}{\queryset_{\mathit{min}}}
\newcommand{\querysetfinite}[1]{\queryset_{\mathit{min}, {#1}}}
\newcommand{\err}{\mathrm{err}}

\newcommand{\errx}[3]{\err_{#1}\left( #2, #3 \right)} %error wrt sample x
\newcommand{\errp}[3]{\err^{#1}\left( #2, #3 \right)} %error wrt dist P
\newcommand{\loss}{L}

\newcommand{\accgame}{\mathsf{Acc}}
\newcommand{\sampaccgame}{\mathsf{SampAcc}}

\title{Algorithmic Stability for Adaptive Data Analysis\thanks{This work unifies and subsumes the two arXiv manuscripts~\cite{BassilySSU15, NissimS15}.}}

\author{
\makebox[1.5in]{\hfill Raef Bassily\thanks{University of California
    San Diego, Center for Information Theory and Applications and
             Department of Computer Science and Engineering. 
Part of this work was done while the author was at Pennsylvania State University, supported by NSF award CDI-0941553. \href{mailto:rbassily@ucsd.edu}{rbassily@ucsd.edu}}\hfill}
\and \makebox[1.5in]{\hfill Kobbi Nissim\thanks{Ben-Gurion University of the Negev, Department of Computer Science and Center for Research on Computation and Society (CRCS), Harvard University. Supported by a grant from the Sloan Foundation, a Simons Investigator grant to Salil Vadhan, and NSF grant CNS-1237235. \href{mailto:kobbi@cs.bgu.ac.il}{kobbi@cs.bgu.ac.il}}\hfill}
\and \makebox[1.5in]{\hfill Adam Smith\thanks{Pennsylvania State University, Department of Computer Science and Engineering. Supported by NSF award IIS-1447700, a Google Faculty
Award and a Sloan Foundation research award. \href{mailto:asmith@psu.edu}{asmith@psu.edu}}\hfill}
\and \makebox[1.5in]{\hfill Thomas Steinke\thanks{Harvard University, John A. Paulson School of Engineering and Applied Sciences. Supported by NSF grants CCF-1116616, CCF-1420938, and CNS-1237235. \href{mailto:tsteinke@seas.harvard.edu}{tsteinke@seas.harvard.edu}} \hfill}
\and \makebox[1.5in]{\hfill Uri Stemmer\thanks{Ben-Gurion University of the Negev, Depaprtment of Computer Science. Supported by the Ministry of Science and Technology, Israel. \href{mailto:stemmer@cs.bgu.ac.il}{stemmer@cs.bgu.ac.il}}  \hfill}
\and \makebox[1.5in]{\hfill Jonathan Ullman\thanks{Northeastern University, College of Computer and Information Science.  Part of this work was done while the author was at Columbia University, supported by a Junior Fellowship from the Simons Society of Fellows.  \href{mailto:jullman@cs.columbia.edu}{jullman@ccs.neu.edu}}\hfill}
}
\date{}

\begin{document}
\maketitle
\pagenumbering{gobble}

\begin{abstract}
Adaptivity is an important feature of data analysis---the choice of
questions to ask about a dataset often depends on previous interactions with the same dataset.  However, statistical validity is typically studied in a nonadaptive model, where all questions are specified before the dataset is drawn.  Recent work by Dwork et al. (STOC, 2015) and Hardt and Ullman (FOCS, 2014) initiated the formal study of this problem, and gave the first upper and lower bounds on the achievable generalization error for adaptive data analysis.

Specifically, suppose there is an unknown distribution $\dist$ and a set of $n$ independent samples $\sample$ is drawn from $\dist$.  We seek an algorithm that, given $\sample$ as input, accurately answers a sequence of adaptively chosen ``queries'' about the unknown distribution $\dist$.  How many samples $n$ must we draw from the distribution, as a function of the type of queries, the number of queries, and the desired level of accuracy?

In this work we make two new contributions towards resolving this question:
\begin{enumerate}
\item We give upper bounds on the number of samples $n$ that are
  needed to answer \emph{statistical queries}. The bounds improve and
  simplify the work of Dwork et al.~(STOC, 2015), and have been applied in
  subsequent work by those authors (\emph{Science}, 2015; NIPS,
  2015).

\item We prove the first upper bounds on the number of samples required to answer more general families of queries.  These include arbitrary \emph{low-sensitivity queries} and an important class of \emph{optimization queries} (alternatively, \emph{risk minimization queries}).
\end{enumerate}

As in Dwork et al., our algorithms are based on a connection with \emph{algorithmic stability} in the form of \emph{differential privacy}. We extend their work by giving a quantitatively optimal, more general, and simpler proof of their main theorem that stable algorithms of the kind guaranteed by differential privacy imply low generalization error. We also show that weaker stability guarantees such as bounded KL divergence and total variation distance lead to correspondingly weaker generalization guarantees.
\end{abstract}

\vfill
\pagebreak

\ifnum\short=0
\tableofcontents
\pagebreak
\fi

\ifnum\final=0
{\color{blue}
To do:
\begin{itemize}
\item 
\end{itemize}
}
\pagebreak
\fi

\pagenumbering{arabic}

\ifnum\short=0
\begin{quotation}
\emph{If you torture the data enough, nature will always confess.}
\hfill -- Ronald Coase
\end{quotation}
\fi
\section{Introduction}

Multiple hypothesis testing is a ubiquitous task in empirical
research.  
A finite sample of data is drawn from some unknown population, and
several analyses are performed on that sample.  
The outcome of an analysis is deemed significant if it is unlikely to
have occurred by chance alone, and a ``false discovery'' occurs if the
analyst incorrectly declares an outcome to be significant.  
False discovery has been identified as a substantial problem in the
scientific community (see e.g.~\cite{Ioannidis05, GelmanL13}).  
This problem persists despite decades of research by statisticians on
methods for preventing false discovery, such as the widely used
Bonferroni Correction~\cite{Bonferroni36,Dunn61} and the
Benjamini-Hochberg Procedure~\cite{BenjaminiH79}.

False discovery is often attributed to misuse of statistics.  
An alternative explanation is that the prevalence of false discovery
arises from the inherent \emph{adaptivity} in the data analysis
process---the fact that the choice of analyses to perform depends on previous
interactions with the data (see e.g.~\cite{GelmanL13}).  
Adaptivity is essentially unavoidable when a sequence of research groups
publish research papers based on overlapping data sets.
Adaptivity also arises naturally in other settings, for example: in
multistage inference algorithms where data are preprocessed (say,
to select features or restrict to a principal subspace) before the
main analysis is performed; in scoring data-based
competitions~\cite{BlumH15}; and in the re-use of holdout or test data~\cite{DworkFHPRR-science-15,DworkFHPRR-nips-2015}.

The general problem of adaptive data analysis was
formally modeled and studied in recent papers by \ifnum\short=1 Dwork et al.~\cite{DworkFHPRR15}\else Dwork, Feldman,
Hardt, Pitassi, Reingold, and Roth~\cite{DworkFHPRR15}\fi~and by Hardt
and Ullman~\cite{HardtU14}.  
The striking results of Dwork et al.~\cite{DworkFHPRR15} gave the
first nontrivial algorithms for provably ensuring statistical validity
in adaptive data analysis, allowing for even an \emph{exponential}
number of tests against the same sample.  In contrast,~\cite{HardtU14,
  SteinkeU14} showed inherent statistical and
computational barriers to preventing false discovery in adaptive settings.

The key ingredient in Dwork et al.~is a notion of ``algorithmic
stability'' that is suitable for adaptive analysis.  
Informally, changing one input to a stable algorithm does not change
it's output ``too much.'' Traditionally, stability was measured via
the change in the generalization error of an algorithm's output, and algorithms stable
according to such a criterion have long been known to
ensure statistical validity in nonadaptive
analysis~\cite{DevroyeW79a,DevroyeW79b,KearnsR99,BousquetE02,SSSS10}.  
Following a connection first suggested by McSherry\footnote{See,
  e.g.,~\cite{McSherry-blog}, although the observation itself dates
  back  at least to 2008 (personal communication).}, Dwork et
al.~showed that a stronger stability condition designed to ensure data
privacy---called \emph{differential
  privacy} (DP)~\cite{DworkMNS06,Dwork06}---guarantees statistical validity in
adaptive data analysis. This allowed them to repurpose known
DP algorithms to prevent false discovery.  
A crucial difference from traditional
notions of stability  is that DP requires a change in one input lead to a small change in the \emph{probability
distribution} on the outputs (in particular, differentially private
algorithms must be randomized).
In this paper, we refer to differential privacy as \emph{max-KL-stability}
(Definition~\ref{def:MKLstable}), to emphasize the relation to the
literature on algorithmic stability other
notions of stability we study (KL- and TV-stability, in particular).

In this work, we extend the results of Dwork et al.~along two axes.
First, we give an \emph{optimal} analysis of the statistical validity of max-KL stable algorithms. As a consequence, we immediately obtain the best known bounds on the \emph{sample complexity} (equivalently, the \emph{convergence rate}) of adaptive data analysis.  Second, we generalize the connection between max-KL stability and statistical validity to a much larger family of statistics.  Our proofs are also significantly simpler than those of Dwork et al., and clarify the role of different stability notions in the adaptive setting.

\subsection{Overview of Results}
\paragraph{Adaptivity and Statistical Queries.}
Following the previous work on this subject (\cite{DworkFHPRR15}% \emph{et sequelae}
), we formalize the problem of adaptive data analysis as follows. There is a \emph{distribution} $\dist$ over some finite universe $\univ$, and a \emph{mechanism} $\mechanism$ that does not know $\dist$, but is given a set $\sample$ consisting of $n$ samples from $\dist$.  Using its sample, the mechanism must answer \emph{queries} on $\dist$.  Here, a query $\query$, coming from some family $Q$, maps a distribution $\dist$ to a real-valued answer.  The mechanism's answer $a$ to a query $\query$ is \emph{$\alpha$-accurate} if $|a-q(\dist)| \leq \alpha$ with high probability.  Importantly, the mechanism's goal is to provide answers that ``generalize'' to the underlying distribution, rather than answers that are specific to its sample.

We model adaptivity by allowing a \emph{data analyst} to ask a sequence of queries $\query_1,\query_2,\dots,\query_{\queries} \in \queryset$ to the mechanism, which responds with answers $a_1, a_2, \dots, a_\queries$. In the adaptive setting, the query $\query_j$ may depend on the previous queries and answers $\query_1,a_1,\dots,\query_{j-1},a_{j-1}$ arbitrarily.  We say the mechanism is \emph{$\alpha$-accurate} given $n$ samples for $\queries$ adaptively chosen queries if, with high probability, when given a vector $\sample$ of $n$ samples from an arbitrary distribution $\dist$, the mechanism accurately responds to any adaptive analyst that makes at most $\queries$ queries.

Dwork et al.~\cite{DworkFHPRR15} considered the family of \emph{statistical queries}~\cite{Kearns93}.  A statistical query $\query$ asks for the expected value of some function on random draws from the distribution.  That is, the query is specified by a function $p \from \univ \to [0,1]$ and its answer is $\query(\dist) = \mathbb{E}_{z \getsr \dist}[p(z)].$

The most natural way to answer a statistical query is to compute the
\emph{empirical answer} $\mathbb{E}_{z \getsr
  \sample}[p(z)],$
which is just the average
value of the function on the given sample $\sample$.\footnote{For
  convenience, we will often use $\sample$ as shorthand for the
  empirical distribution over $\sample$. We use $z \getsr \sample$ to mean a random element chosen from the uniform distribution over the elements of $\sample$.}  
It is simple to show that when $k$ queries are specified
\emph{nonadaptively} (i.e.~independent of previous answers), then
the empirical answer is within $\query(\dist)\pm\alpha$ (henceforth, ``$\alpha$-accurate'') with high
probability so long as the sample has size
$n \gtrsim \log(\queries)
/ \alpha^2$.~\footnote{This guarantee follows from bounding the error of each query using a
  Chernoff bound and then taking a union bound over all queries. The
  $\log \queries$ term corresponds to the Bonferroni correction in
  classical statistics.}
However, when the queries can be chosen adaptively, the empirical
average performs much worse.  
In particular, there is an algorithm (based on \cite{DinurN03}) that, after seeing the empirical
answer to $k = O(\alpha^2 n)$ random queries, can find a query such
that the empirical answer and the correct answer differ by $\alpha$.  
Thus, the empirical average is only guaranteed to be accurate when $n
\gtrsim \queries / \alpha^2$, and so exponentially more samples are required to guarantee accuracy when the queries may be adaptive.

\paragraph{Answering Adaptive Statistical Queries.}
Surprisingly, Dwork et al.~\cite{DworkFHPRR15}, showed there are
mechanisms that are much more effective than na\"ively outputting the
empirical answer.  
They show that ``stable'' mechanisms are accurate
and,  by applying a stable mechanism from the literature on
differential privacy, they obtain a mechanisms that are accurate given
only $n \gtrsim \sqrt{\queries} / \alpha^{2.5}$ samples, which is a
significant %exponential
improvement over the na\"ive
mechanism when $\alpha$ is not too small.
(See Table~\ref{fig:results} for more detailed statements of their results, including results that achieve an \emph{exponential} improvement in the sample complexity when $|\univ|$ is bounded.)

Our first contribution is to give a simpler and quantitatively optimal
analysis of the generalization properties of stable algorithms, which immediately yields new
accuracy bounds for adaptive statistical queries.  
In particular, we show that $n \gtrsim \sqrt{k} / \alpha^2$ samples
suffice.
Since $1/\alpha^2$ samples are required to answer a single nonadaptive
query, our dependence on $\alpha$ is optimal.

\paragraph{Beyond Statistical Queries.}
Although statistical queries are surprisingly general~\cite{Kearns93}, we would like to be able to ask more general queries on the distribution $\dist$ that capture a wider variety of machine learning and data mining tasks.  To this end, we give the first bounds on the sample complexity required to answer large numbers of adaptively chosen \emph{low-sensitivity queries} and \emph{optimization queries,} which we now describe.

Low-sensitivity queries are a generalization of statistical queries. A
query is specified by an arbitrary function $p \from \univ^n \to \R$
satisfying $|p(\sample) - p(\sample')| \leq 1/n$ for every $\sample,
\sample' \in \univ^n$ differing on exactly one element.  The query
applied to the population is defined to be $\query(\dist) =
\mathbb{E}_{\sample \getsr \dist^n}[p(\sample)]$. Examples include
\emph{distance queries} (e.g. ``How far is the sample from being
well-clustered?'') and maxima of statistical queries (e.g. ``What is the
classification error of the best $k$-node decision tree?'')

Optimization queries are a broad generalization of low-sensitivity
queries to arbitrary output domains.  The query is specified by a loss
function $\loss \from \univ^n \times \Theta \to \R$ that is
low-sensitivity in its first parameter, and the goal is to output
$\theta \in \Theta$ that is ``best'' in the sense that it minimizes
the average loss.  Specifically, $\query(\dist) = \argmin_{\theta \in
  \Theta} \mathbb{E}_{z \getsr \dist^n}[\loss(z; \theta)].$ 
An important special case is when $\Theta\subseteq \R^d$ is convex and
$\loss$ is convex in $\theta$, which captures many fundamental
regression and classification problems.

Our sample complexity bounds are summarized in Table~\ref{fig:results}.
\begin{table}[ht!]\small
\begin{center}
%\begin{tabularx}{0.95\textwidth}{|*{4}{Y|}}
\begin{tabular}{|c|c|c|c|}
\hline
\multirow{2}{*}{Query Type} &\multicolumn{2}{c|}{Sample Complexity}& \multirow{2}{*}{Time per Query}\\
\cline{2-3}
             &\cite{DworkFHPRR15}       & This Work & \\
\hline
Statistical ($\queries \ll n^2$) & $ \tilde{O}\left(\dfrac{\sqrt{\queries}}{\alpha^{2.5}} \right)$ & $ \tilde{O}\left(\dfrac{\sqrt{k}}{\alpha^2}\right)$ & $\poly(n, \log |\univ|)$\\
\hline
Statistical ($\queries \gg n^2$) & $ \tilde{O}\left( \dfrac{\sqrt{\log |\univ|} \cdot \log^{3/2} k}{\alpha^{3.5}} \right)$ & $ \tilde{O}\left(\dfrac{\sqrt{\log |\univ|} \cdot \log k}{\alpha^3}\right)$ & $\poly(n, |\univ|)$\\
\hline
Low Sensitivity ($\queries \ll n^2$) & --- & $ \tilde{O}\left(\dfrac{\sqrt{k}}{\alpha^2}\right)$ & $ \poly(n, \log |\univ|)$\\
\hline
Low Sensitivity ($\queries \gg n^2$) & --- & $ \tilde{O}\left(\dfrac{\log |\univ| \cdot \log k}{\alpha^3}\right)$ & $\poly(|\univ|^n)$\\
\hline
Convex Min. ($\queries \ll n^2$) & --- & $\tilde{O}\left(\dfrac{\sqrt{d k}}{\alpha^2}\right)$ & $\poly(n, d, \log |\univ|)$\\
\hline
Convex Min. ($\queries \gg n^2$) & --- & $ \tilde{O}\left(\dfrac{ (\sqrt{d }+  \log k ) \cdot \sqrt{\log |\univ|}}{\alpha^3}\right)$ & $\poly(n, d, |\univ|)$\\
\hline
\end{tabular}
\caption{\label{fig:results}  Summary of Results. Here $k = $ number
  of queries, $n = $ number of samples, $\alpha = $ desired accuracy,
  $\univ = $ universe of possible samples, $d = $ dimension of
  parameter space $\Theta$.}
\end{center}
\end{table}

\ifnum\short=1
\vspace{-2em}
\fi

\paragraph{Subsequent Work.} Our bounds were
applied in subsequent work of Dwork et al.
\cite{DworkFHPRR-science-15,DworkFHPRR-nips-2015} in the analysis of their
``reusable holdout'' construction.

\subsection{Overview of Techniques}
Our main result is a new proof, with optimal parameters, that a stable
algorithm that provides answers to adaptive queries that are close to
the empirical value on the sample gives answers that generalize to the
underlying distribution. In particular, we prove:%for the class of low-sensitivity queries, we prove the following.

\begin{thm} [Main ``Transfer Theorem''] \label{thm:Transfer-intro} %\label{thm:LowSensTransfer}
Let $\mechanism$ be a mechanism that takes a sample $\sample \in \univ^n$ and answers $k$ adaptively-chosen low-sensitivity queries.  Suppose that $\mechanism$ satisfies the following:
\begin{enumerate}
\item For every sample $\sample$, $\mechanism$'s answers are $(\alpha, \alpha \beta)$-accurate\footnote{Accuracy is formally defined in Section \ref{sec:mechanisms}. Informally, a mechanism is $(\alpha,\beta)$-accurate if every answer it produces to an adversarial adaptive sequence of queries is $\alpha$-accurate with probability at least $1-\beta$.}
  with respect to  the sample $\sample$. That is, 
$\textstyle
 \pr{}{\max_{j \in k} \left|q_j(\sample)-a_j\right| \leq \alpha} \geq
1- \alpha \beta,$
where $q_1, \cdots, q_k : \univ^n \to \mathbb{R}$ are the
low-sensitivity queries that are asked and $a_1, \cdots, a_k \in
\mathbb{R}$ are the answers given. The probability is taken only over $\mechanism$'s random coins.
\item %The mechanism 
$\mechanism$ satisfies $(\alpha, \alpha \beta)$-max-KL stability (Definition~\ref{def:MKLstable},  identical to $(\alpha,\alpha \beta)$-differential privacy).
\end{enumerate}
Then, if $\sample$ consists of $n$ samples from an arbitrary distribution $\dist$ over $\univ$, %and $\mechanism$ answers $k$ adaptively chosen low-sensitivity queries, 
$\mechanism$'s answers are $(O(\alpha), O(\beta))$-accurate with
respect to $\dist.$  That is, 
$\textstyle\pr{}{\max_{j \in k} \left|q_j(\dist) - a_j\right| \leq
  O(\alpha)} \geq 1-O(\beta),$
where the probability is taken only over the choice of $\sample \getsr \dist^n$ and $\mechanism$'s random coins.
 \end{thm}

Our actual result is somewhat more general than Theorem~\ref{thm:Transfer-intro}.  We
show that the population-level error of a stable algorithm is close
to its error on the sample, whether or not that error is low. Put
glibly: stable algorithms cannot be wrong without realizing it.

Compared to the results of \cite{DworkFHPRR15}, Theorem \ref{thm:Transfer-intro} requires a quantitatively weaker stability guarantee---$(\alpha,\alpha \beta)$-stability, instead of $(\alpha,(\beta/k)^{1/\alpha})$-stability.  It also applies to arbitrary low-sensitivity queries as opposed to the special case of statistical queries.

Our analysis differs from that of Dwork et al.~in two key ways. First,
we give a better bound on the probability with which a \emph{single}
low-sensitivity query output by a max-KL stable algorithm has good
generalization error. Second, we show a reduction from the case of
\emph{many} queries to the case of a single query that has no loss in
parameters (in contrast, previous work took a union bound over
queries, leading to a dependence on $k$, the number of queries). 

Both steps rely on a thought experiment in which several
``real'' executions of a stable algorithm are simulated inside another
algorithm, called a \emph{monitor}, which outputs a function of the
``real'' transcripts. Because stability is closed under
post-processing, the monitor is itself stable. Because
it exists only as a thought experiment, the monitor can be given
knowledge of the true distribution from which the data are drawn, and
can use this knowledge to process the outputs of the simulated
``real'' runs.  The monitor technique allows us to start from a
basic guarantee, which states that a single query has good
generalization error with constant probability, and amplify the
guarantee so that
\begin{enumerate*}[label=(\emph{\alph*})]
\item the generalization error holds with very high probability, and
\item the guarantee holds simultaneously over all queries in a
  sequence.
\end{enumerate*}
The proof of the basic guarantee follows the lines of existing proofs
using algorithmic stability (e.g., \cite{DevroyeW79a}), while the
monitor technique and the resulting amplification statements are new.

 The amplification of success probability is the more
technically sophisticated of the two key steps. The idea is to run many (about
$1/\beta$, using the notation of
Theorem~\ref{thm:Transfer-intro}) copies of a stable mechanism on
independently selected data sets. Each of these interactions results
in a sequence of queries and answers. The monitor then selects the the query and answer pair from amongst all of the sequences that has the largest error. It then outputs this query as well as the
index of the interaction that produced it. Our main technical lemma shows that the monitor will find a
``bad'' query/dataset pair (one where the true and empirical values of
the query differ) with at most constant probability. This implies that the each of
the real executions outputs a bad query with probability
$O(\beta)$. Relative to previous work, the resulting argument yields
better bounds, applies to more general classes of queries, and even
%makes you a fresh cup of coffee while you are reading the proof
generalizes to other notions of stability.

\ifnum\short=1
\paragraph{Optimality and Computational Complexity.}
In general, we do not know that our bounds are optimal.  The current best lower bounds on the sample complexity of answering adaptively chosen statistical queries are $n \gtrsim \log(\queries)/\alpha^2$ (which applies even for non-adaptive SQs) and $n \gtrsim \min\{\sqrt{\queries}, \sqrt{\log |\univ|}\} / \alpha$~\cite{HardtU14, SteinkeU14}.  However, we do know that
\begin{enumerate*}[label=(\emph{\alph*})]
\item our connection between max-KL stability and generalization is optimal, and 
\item for every setting we consider, the corresponding max-KL stable algorithm is optimal~\cite{BunUV14, BassilyST14}.
\end{enumerate*}
Thus, significantly different techniques will be required to improve our results.

We also remark that, while several of our algorithms have an
undesirable polynomial dependence on $|\univ|$ in their running times
(since the input is a dataset of $n \log |\univ|$ bits, a
computationally efficient algorithm would run in time polynomial in
$\log |\univ|$), this is known to be inherent under the widely
believed assumption that exponentially hard one-way functions
exist~\cite{HardtU14, SteinkeU14}.

\else

\paragraph{Optimality.}
In general, we cannot prove that our bounds are optimal.  Even for nonadaptive statistical queries, $n \gtrsim \log(\queries) / \alpha^2$ samples are necessary, and~\cite{HardtU14, SteinkeU14} showed that $n \gtrsim \min\{\sqrt{\queries}, \sqrt{\log |\univ|}\} / \alpha$ samples are necessary to answer adaptively chosen statistical queries.  However, the gap between the upper and lower bounds is still significant.

However, we can show that our connection between max-KL stability and generalization is optimal (see Section~\ref{sec:opt} for details). Moreover, for every family of queries we consider, no max-KL stable algorithm can achieve better sample complexity~\cite{BunUV14, BassilyST14}. Thus, any significant improvement to our bounds must come from using a weaker notion of stability or some entirely different approach.

\paragraph{Computational Complexity.}
Throughout, we will assume that the analyst only issues queries $\query$ such that the empirical answer $\query(\sample)$ can be evaluated in time $\poly(n, \log |\univ|).$  When $\queries \ll n^2$ our algorithms have similar running time.  However, when answering $\queries \gg n^2$ queries, our algorithms suffer running time at least $\poly(n, |\univ|)$.  Since the mechanism's input is of size $n \cdot \log |\univ|$, these algorithms cannot be considered computationally efficient.  For example, if $\univ = \zo^{d}$ for some dimension $d$, then in the non adaptive setting $\poly(n, d)$ running time would suffice, whereas our algorithms require $\poly(n, 2^d)$ running time.  Unfortunately, this running time is known to be optimal, as~\cite{HardtU14, SteinkeU14} (building on hardness results in privacy~\cite{Ullman13}) showed that, assuming exponentially hard one-way functions exist, any $\poly(n, 2^{o(d)})$ time mechanism that answers $\queries = \omega(n^2)$ statistical queries is not even $1/3$-accurate.

\paragraph{Stable / Differentially Private Mechanisms.}
Each of our results requires instantiating the mechanism with a suitable stable / differentially private algorithm.  For statistical queries, the optimal mechanisms are the well known Gaussian and Laplace Mechanisms (slightly refined by~\cite{SteinkeU15}) when $k$ is small and the Private Multiplicative Weights Mechanism~\cite{HardtR10} when $\queries$ is large.  For arbitrary low-sensitivity queries, the Gaussian or Laplace Mechanism is again optimal when $k$ is small, and for large $\queries$ we can use the Median Mechanism~\cite{RothR10}.

When considering arbitrary search queries over an arbitrary finite
range, the optimal algorithm is the Exponential
Mechanism~\cite{McSherryT07}.  For the special case of convex
minimization queries over an infinite domain, we use the optimal
algorithm of \cite{BassilyST14} when $k$ is small, and when $k$ is large, we  use an algorithm of~\cite{Ullman15} that accurately answers exponentially many such queries.

\paragraph{Other Notions of Stability} Our techniques applies to
notions of distributional stability other than max-KL/differential
privacy. In particular, defining stability in terms of total variation
(TV) or KL divergence (KL) leads to bounds on the 
generalization error that have polynomially, rather than
exponentially, decreasing tails. See
Section~\ref{sec:othernotions-bounds} for details.

\fi

\ifnum\short=1
\paragraph{This Abstract.} 
We present here our core technical claims,
along with some proofs. The full version of the paper, appended, contains full proofs along with additional results and
applications. 
\fi

\section{Preliminaries}
\subsection{Queries}
Given a distribution $\dist$ over $\univ$ or a sample $\sample =
(\sample_1, \cdots, \sample_n) \in \univ^n$, we would like to answer
\emph{queries} about $\dist$ or $x$ from some family $\queryset$.  We
will often want to bound the ``sensitivity'' of the queries with
respect to changing one element of the sample.  To this end, we use $x
\sim x'$ to denote that $x, x' \in \univ^n$ differ on at most one
entry.  We will consider several different families of queries:%, listed in order of increasing generality:
\begin{itemize}[leftmargin=1em]
\ifnum\short =0
\item
\textbf{Statistical Queries:} These queries are specified by a function $\query \from \univ \to [0,1]$, and (abusing notation) are defined as $$\query(\dist) = \ex{z \getsr \dist}{\query(z)} \eqand \query(\sample) = \frac{1}{n} \sum_{i\in [n]} \query(\samp_i).$$
The error of an answer $a$ to a statistical query $q$ with respect to $\dist$ or $\sample$ is defined to be $$\errx{\sample}{q}{a} = a-q(\sample) \eqand \errp{\dist}{q}{a} = a-q(\dist).$$
\fi

\item \textbf{$\Delta$-Sensitive Queries:} For $\Delta \in [0,1]$, $n
  \in \N$, these queries are specified by a function $\query \from
  \univ^n \to \R$ satisfying $|q(x)-q(x')| \leq \Delta$ for every pair
  $x, x' \in \univ^n$ differing in only one entry. Abusing notation, let $$\query(\dist) = \ex{z \getsr \dist^n}{\query(z)}.$$
The error of an answer $a$ to a $\Delta$-sensitive query $q$ with respect to $\dist$ or $\sample$ is defined to be $$\errx{\sample}{q}{a} = a-q(\sample) \eqand \errp{\dist}{q}{a} = \ex{z \getsr \dist^n}{\errx{z}{q}{a}} = a-q(\dist).$$ We denote the set of all $\Delta$-sensitive queries by $\querysetDelta$.  If $\Delta = O(1/n)$ we say the query is \emph{low sensitivity}.
\ifnum\short =0
Note that $1/n$-sensitive queries are a strict generalization of statistical queries. 
\else
A special case of $1/n$-sensitive queries is \textbf{statistical queries}. A statistical query is specified by a function $\query \from \univ \to [0,1]$ and defined by $\query(\sample) = \frac{1}{n} \sum_{i\in [n]} \query(\samp_i)$, whence $\query(\dist) = \ex{z \getsr \dist}{\query(z)}$.
\fi

\item \textbf{Minimization Queries:} These queries are specified by a
  loss function $\loss \from \univ^n \times \Theta \to \R$. We require that  $\loss$ has sensitivity $\Delta$ with respect to its
  first parameter, that is, $$\sup_{\theta \in \Theta, \ x,x' \in
    \univ^n, \ x \sim x'} |L(x;\theta)-L(x';\theta)|\leq \Delta \, . $$
 Here
  $\Theta$ is an arbitrary set of items (sometimes called ``parameter
  values'') among which we aim to chose the item (``parameter'') with minimal
  loss, either with
  respect to a particular input data set $x$, or with respect to
  expectation over a distribution $\dist$.

\ifnum\short =0
The error of an answer $\theta \in \Theta$ to a minimization query  $\loss \from \univ^n \times \Theta \to \R$ with respect to $\sample$ is defined to be $$\errx{\sample}{\loss}{\theta} = \loss(\sample,\theta) - \min_{\theta^* \in \Theta} \loss(\sample,\theta^*)$$ and, with respect to $\dist$, is $$\errp{\dist}{\loss}{\theta} = \ex{z \getsr \dist^n}{\errx{z}{\loss}{\theta}}= \ex{z \getsr \dist^n}{\loss(z,\theta)} -  \ex{z \getsr \dist^n}{\min_{\theta^* \in \Theta} \loss(z,\theta^*)}.$$
Note that $\min_{\theta^* \in \Theta} \ex{z \getsr \dist^n}{\loss(z,\theta^*)} \geq \ex{z \getsr \dist^n}{\min_{\theta^* \in \Theta} \loss(z,\theta^*)}$, whence $$\ex{z \getsr \dist^n}{\loss(z,\theta)} -  \min_{\theta^* \in \Theta} \ex{z \getsr \dist^n}{ \loss(z,\theta^*)} \leq \errp{\dist}{\loss}{\theta} .$$

Note that minimization queries (with $\Theta = \R$) generalize low-sensitivity queries: Given a $\Delta$-sensitive $q : \univ^n \to \mathbb{R}$, we can define $\loss(\sample;\theta) = |\theta-q(\sample)|$ to obtain a minimization query with the same answer.

We denote the set of minimization queries by $\querysetMIN$.
We highlight two  special cases:
\begin{itemize}
\item \emph{Minimization for Finite Sets:} We denote by $\querysetfinite{D}$
  the set of minimization queries where $\Theta$ is finite with size at
  most $D$.
\item \emph{Convex Minimization Queries:} If $\Theta\subset \R^d$ is closed and
convex and $L(\sample;\cdot)$ is convex on $\Theta$ for every data set $\sample$, then the query can be answered nonprivately up to any desired error $\alpha$, in time polynomial in $d$ and $\alpha$.  We denote the set
  of all convex minimization queries by $\querysetCM$.
\end{itemize}
\else
The error of an answer $\theta \in \Theta$ to a minimization query  $\loss \from \univ^n \times \Theta \to \R$ with respect to a sample $\sample$ or distribution $\dist$ is defined to be $$\errx{\sample}{\loss}{\theta} = \loss(\sample,\theta) - \min_{\theta^* \in \Theta} \loss(\sample,\theta^*) \qquad\text{and}\qquad \errp{\dist}{\loss}{\theta} = \ex{z \getsr \dist^n}{\errx{z}{\loss}{\theta}}.$$ 
Note that $\ex{z \getsr \dist^n}{\loss(z,\theta)} -  \min_{\theta^* \in \Theta} \ex{z \getsr \dist^n}{ \loss(z,\theta^*)} \leq \errp{\dist}{\loss}{\theta}.$ We denote the set of minimization queries by $\querysetMIN$.
\fi

\end{itemize}

\subsection{Mechanisms for Adaptive Queries} \label{sec:mechanisms}
Our goal is to design a \emph{mechanism} $\mechanism$ that answers queries on $\dist$ using only independent samples $\samp_{1},\dots,\samp_{n} \getsr \dist$.  Our focus is the case where the queries are chosen adaptively and adversarially.

Specifically, $\mechanism$ is a stateful algorithm that holds a collection of samples $\samp_1,\dots,\samp_{n} \in \univ$, takes a query $\query$ from some family $\queryset$ as input, and returns an answer $a$.  We require that when $\samp_1,\dots,\samp_{n}$ are independent samples from $\dist$, the answer $a$ is ``close'' to $\query(\dist)$ in a sense that is appropriate for the family of queries.  Moreover we require that this condition holds for every query in an adaptively chosen sequence $\query_{1},\dots, \query_{\queries}$.  Formally, we define an accuracy game between a mechanism $\mechanism$ and a stateful \emph{data analyst} $\adv$ in Figure \ref{fig:accgame1}.
\begin{figure}[ht!]
\begin{framed}
\begin{algorithmic}
\STATE{$\adv$ chooses a distribution $\dist$ over $\univ$.}
\STATE{Sample $\samp_1,\dots,\samp_{n} \getsr \dist$, let $\sample = (\samp_1,\dots,\samp_{n})$. (Note that $\adv$ does not know $\sample$.)}
\STATE{For $j = 1,\dots,\queries$}
\INDSTATE[1]{$\adv$ outputs a query $\query_{j} \in \queryset$.}
\INDSTATE[1]{$\mechanism(\sample, \query_{j})$ outputs $a_{j}$.}
\INDSTATE[1]{(As $\adv$ and $\mechanism$ are stateful, $q_j$ and $a_j$ may depend on the history $q_1,a_1,\dots,q_{j-1},a_{j-1}$.)}
\end{algorithmic}
\end{framed}
\vspace{-6mm}
\caption{The Accuracy Game $\accgame_{n, \queries, \queryset}[\mechanism, \adv]$ \label{fig:accgame1}}
\end{figure}

\newcommand{\transcript}{\mathsf{Int}}

\begin{definition}[Accuracy] \label{def:accuratemechanism}
A mechanism $\mechanism$ is \emph{$(\alpha,\beta)$-accurate with respect to the population for $\queries$ adaptively chosen queries from $\queryset$ given $n$ samples in $\univ$} if for every adversary $\adv$,
$$
\pr{\accgame_{n, \queries, \queryset}[\mechanism, \adv]}{
\max_{ j \in [\queries]} ~ \left| \errp{\dist}{q_j}{a_j} \right|\leq \alpha} \geq 1 - \beta.
$$
\end{definition}

We will also use a definition of accuracy relative to the sample given to the mechanism\ifnum\short=1.\else, described in Figure~\ref{fig:sampleaccgame1}.\fi

\begin{figure}[ht!]
\begin{framed}
\begin{algorithmic}
\STATE{$\adv$ chooses $\sample = (\sample_1,\dots,\sample_{n}) \in  \univ^n$.}
\STATE{For $j = 1,\dots,\queries$}
\INDSTATE[1]{$\adv$ outputs a query $\query_{j} \in \queryset$.}
\INDSTATE[1]{$\mechanism(\sample, \query_{j})$ outputs $a_{j}$.}
\INDSTATE[1]{($q_j$ and $a_j$ may depend on the history $q_1,a_1,\dots,q_{j-1},a_{j-1}$ and on $\sample$.)}
\end{algorithmic}
\end{framed}
\vspace{-6mm}
\caption{The Sample Accuracy Game $\sampaccgame_{n, \queries, \queryset}[\mechanism, \adv]$\label{fig:sampleaccgame1}}
\end{figure}

\begin{definition}[Sample Accuracy] \label{def:accuratemechanism-samp}
A mechanism $\mechanism$ is \emph{$(\alpha,\beta)$-accurate with respect to samples of size $n$ from $\univ$ for $\queries$ adaptively chosen queries from $\queryset$} if for every adversary $\adv$,
$$
\pr{\sampaccgame_{n, \queries, \queryset}[\mechanism, \adv]}{\max_{ j \in [\queries] }~~~ \left| \errx{\sample}{q_j}{a_j} \right| \leq \alpha} \geq 1-\beta.
$$
\end{definition}

\subsection{Max-KL Stability (a.k.a. Differential Privacy)}\label{subsec:max-kl} 
Informally, an algorithm is ``stable'' if changing one of its inputs does not change its output ``too much.''  For our results, we will consider randomized algorithms, and require that changing one input does not change the \emph{distribution} of the algorithm's outputs too much.  With this in mind, we will define here one notion of algorithmic stability that is related to the well-known notion of KL-divergence between distributions. \longshort{In Section~\ref{sec:other-notions}, we will give}{In the full version of this paper, we consider} other related notions of algorithmic stability based on different notions of closeness between distributions.

\begin{definition}[Max-KL Stability]
\label{def:MKLstable}
Let $\alg \from \univ^n \to \cR$ be a randomized algorithm.  We say that $\alg$ is \emph{$(\eps, \delta)$-max-KL stable} if for every pair of samples $\sample, \sample'$ that differ on exactly one element, and every $R \subseteq \cR$, 
$$
\pr{}{\alg(\sample) \in R} \leq e^{\eps} \cdot \pr{}{\alg(\sample') \in R} + \delta.
$$
\end{definition}
This notion of $(\eps, \delta)$-max-KL stability is also commonly known as \emph{$(\eps, \delta)$-differential privacy}~\cite{DworkMNS06}, however in this context we choose the term max-KL stability to emphasize the conceptual relationship between this notion and other notions of algorithmic stability that have been studied in machine learning. We also emphasise that our work has a very different motivation to the motivation of differential privacy --- stable algorithms are desirable even when privacy is not a concern, such as when the data does not concern humans.

In our analysis, we will make crucial use of the fact that max-KL-stability \longshort{(as well as the other notions of stability discussed in Section~\ref{sec:other-notions})}{}is \emph{closed under post-processing}.

\begin{lemma}[Post-Processing] \label{lem:postprocessing}
Let $\alg \from \univ^n \to \cR$ and $f \from \cR \to \cR'$ be a pair of randomized algorithms.  If $\alg$ is $(\eps, \delta)$-max-KL-stable then the algorithm $f(\alg(\sample))$ is $(\eps, \delta)$-max-KL-stable.
\end{lemma}

\ifnum\short = 0
\subsubsection{Stability for Interactive Mechanisms}
\label{subsec:interact-stability}
The definition we gave above does not immediately apply to algorithms that interact with a data analyst to answer adaptively chosen queries.  Such a mechanism does not simply take a sample $\sample$ as input and produce an output.  Instead, in the interactive setting, there is a mechanism $\mechanism$ that holds a sample $\sample$ and interacts with some algorithm $\adv.$  We can view this entire interaction between $\mechanism$ and $\adv$ as a single noninteractive meta algorithm that outputs the transcript of the interaction and define stability with respect to that meta algorithm.  Specifically, we define the algorithm $\alg[\mechanism, \adv](x)$ that simulates the interaction between $\mechanism(x)$ and $\adv$ and outputs the messages sent between them.  The simulation is also parameterized by $n, \queries, \queryset,$ although we will frequently omit these parameters when they are clear from context.  
\begin{figure}[ht!]
\begin{framed}
\begin{algorithmic}
\STATE{\begin{center}$\alg_{n, \queries, \queryset}[\mechanism, \adv] \from \univ^n \to (\queryset \times \cR)^k$.\end{center}}
\INDSTATE[0]{\textbf{Input:} A sample $\sample \in \univ^n$}
\INDSTATE[1]{For $j = 1,\dots,\queries$}
\INDSTATE[2]{Feed $a_{j-1}$ to $\adv$ and get a query $\query_{j} \in \queryset.$}
\INDSTATE[2]{Feed $\query_{j}$ to $\mechanism(\sample)$ and get an answer $a_{j} \in \cR.$}
\INDSTATE[1]{Output $( (\query_{1}, a_{1}),\dots, (\query_{\queries}, a_{\queries})).$}
\end{algorithmic}
\end{framed}
%\vspace{-6mm}
%\caption{The Accuracy Game $\accgame_{n, \queries, \queryset}[\mechanism, \adv]$ \label{fig:accgame}}
\end{figure}

Note that $\alg[\mechanism, \adv]$ is a noninteractive mechanism, and its output is just the query-answer pairs of $\mechanism$ and $\adv$ in the sample accuracy game, subject to the mechanism being given the sample $\sample.$  Now we can define the stability of an interactive mechanism $\mechanism$ using $\alg.$
\begin{definition}\label{def:interact-stability}[Stability of for Interactive Mechanism]
We say an interactive mechanism $\mechanism$ is \emph{$(\eps, \delta)$-max-KL stable for $\queries$ queries from $\queryset$} if for every adversary $\adv$, the algorithm $\alg_{n, \queries, \queryset}[\mechanism, \adv](\sample) \from \univ^n \to (\queryset \times \cR)^{k}$ is $(\eps, \delta)$-max-KL stable.
\end{definition}
\else
\paragraph{Stability for Interactive Mechanisms}
The definition we gave above does not immediately apply to algorithms that interact with a data analyst to answer adaptively chosen queries.  Such a mechanism does not simply take a sample $\sample$ as input and produce an output.  Instead, in the interactive setting, there is a mechanism $\mechanism$ that holds a sample $\sample$ and interacts with some algorithm $\adv.$  We can view this entire interaction between $\mechanism$ and $\adv$ as a single noninteractive algorithm that outputs the transcript of the interaction and define stability with respect to that meta algorithm.  Specifically, we define the algorithm $\alg[\mechanism, \adv](x)$ that simulates the interaction between $\mechanism(x)$ and $\adv$ and outputs the messages sent between them. We say an interactive mechanism $\mechanism$ is \emph{$(\eps, \delta)$-max-KL stable for $\queries$ queries from $\queryset$} if for every adversary $\adv$, the algorithm $\alg[\mechanism, \adv](\sample) \from \univ^n \to (\queryset \times \cR)^{k}$ is $(\eps, \delta)$-max-KL stable.
\fi

\ifnum\short = 0
\subsubsection{Composition of Max-KL Stability}
\else
\paragraph{Composition of Max-KL Stability}
\fi
The definition above allows for \emph{adaptive composition.} This follows directly from composition results of \emph{$(\eps, \delta)$-differentially private} algorithms. A mechanism that is $(\eps, \delta)$-max-KL stable for $1$ query is $(\approx \eps \sqrt{k}, \approx \delta k)$-stable for $k$ adaptively chosen queries~\cite{DworkMNS06, DworkRV10}. More precisely, for every $0 \leq \eps \leq 1$ and $\delta, \delta' > 0$, if a mechanism that is $(\eps, \delta)$-max-KL stable for $1$ query is used to answer $k$ adaptively chosen queries, it remains $(\eps \sqrt{k\log(1/\delta')} + 2\eps^2 k,\; \delta' + k\delta)$-max-KL stable~\cite{DworkRV10}.

\section{From Max-KL Stability to Accuracy for Low-Sensitivity Queries}
In this section we prove our main result that any mechanism that is both accurate with respect to the sample and satisfies max-KL stability (with suitable parameters) is also accurate with respect to the population.  The proof proceeds in two mains steps.  First, we prove a lemma that says that there is no max-KL stable mechanism that takes several independent sets of samples from the distribution and finds a query and a set of samples such that the answer to that query on that set of samples is very different from the answer to that query on the population.  \longshort{In Section~\ref{sec:SQCE} we prove this lemma for the simpler case of statistical queries and then in~\ref{sec:lowsensCE} we extend the proof to the more general case of low-sensitivity queries.}{}

The second step is to introduce a \emph{monitoring algorithm}. This
monitoring algorithm will simulate the interaction between the
mechanism and the adversary on multiple independent sets of samples. 
\rcomm{(I think it is better to say explicitly that the monitoring
  algorithm knows the distribution $\dist$. We don't seem to state
  this explicitly except that we denote the monitoring as
  $\mathcal{W}_{\dist}$.)} 
It will then output the least accurate query across all the different interactions.  We show that if the mechanism is stable then the monitoring algorithm is also stable.  By choosing the number of sets of samples appropriately, we ensure that if the mechanism has even a small probability of being inaccurate in a given interaction, then the monitor will have a constant probability of finding an inaccurate query in one of the interactions.  By the lemma proven in the first step, no such monitoring algorithm can satisfy max-KL stability, therefore every stable mechanism must be accurate with high probability.
%\subsection{From max-KL Stability to Accuracy}
%We can extend the results of the previous section to show that algorithms that satisfy the stronger notion of max-KL stability imply high confidence accuracy bounds.

\newcommand{\trials}{T}
\newcommand{\trial}{t}
\newcommand{\samples}{\mathbf{X}}

\ifnum\short=0
\subsection{Warmup: A Single-Sample De-Correlated Expectation Lemma for SQs} \label{sec:singleSQCE}
As a warmup, in this section we give a simpler version of our main lemma for the case of statistical queries and a single sample.  Although these results follow from the results of Section~\ref{sec:lowsensCE} on general low-sensitivity queries, we include the simpler version to introduce the main ideas in the cleanest possible setting.

\begin{lem} \label{lem:SKLCondExpSQ}
Let $\monitor \from (\univ^{n})^{\trials} \to \queryset$ be $(\eps,\delta)$-max-KL stable where $\queryset$ is the class statistical queries $\query : \univ \to [0,1]$. Let $\dist$ be a distribution on $\univ$ and let $\sample \getsr \dist^n$. Then\footnote{The notation $\ex{\sample,\monitor}{\query(\dist) \mid \query = \monitor(\sample)}$ should be read as ``the expectation of $\query(\dist)$, where $\query$ denotes the output of $\monitor(\sample)$.'' That is, the ``event'' being conditioned on is simply a definition of the random variable $\query$. }
$$
\left|\ex{\sample,\monitor}{\query(\dist) \mid \query = \monitor(\sample)} - \ex{\samples,\monitor}{\query(\sample) \mid \query = \monitor(\sample)} \right| \leq e^{\eps} - 1 + \delta.
$$
\end{lem}

\begin{proof}[Proof of Lemma~\ref{lem:SKLCondExpSQ}]
Before giving the proof, we set up some notation.  Let $\sample = (\samp_{1},\dots,\samp_{n})$.  For a single element $\samp' \in \univ$, and an index $i \in [n]$, we use $\sample_{i \rightarrow \samp'}$ to denote the new sample where the $i$-th element of $\sample$ has been replaced by the element $\samp'$. Let $\samp' \getsr \dist$ be independent from $\sample$.

We can now calculate
\begin{align*}
&\ex{\sample, \monitor}{\query(\sample) \mid \query = \monitor(\sample)} \\
={} &\frac{1}{n} \sum_{i = 1}^{n} \ex{\sample, \monitor}{\left. \query(\samp_{i}) \; \right| \; \query = \monitor(\sample)} \\
={} &\frac{1}{n} \sum_{i = 1}^{n}  \int_0^1 \pr{\sample, \monitor}{\left. \query(\samp_{ i}) \; > z\right| \; \query = \monitor(\sample)} \mathrm{d}z\\
%\end{align*}
\intertext{Now we can apply max-KL stability:}
%\begin{align*}
\leq{} &\frac{1}{n} \sum_{i = 1}^{n} \int_0^1 e^\eps \pr{\sample, \monitor}{\left.  \query(\samp_{ i}) \; > z\right| \; \query ={} \monitor(\sample_{i \rightarrow \samp'})} +\delta\mathrm{d}z \tag{by $(\eps, \delta)$-max-KL stability}\\
=& \frac{1}{n} \sum_{i = 1}^{n} \left( e^{\eps} \cdot \ex{\samp', \sample, \monitor}{\left.  \query(\samp_{i}) \; \right| \; \query = \monitor(\sample_{i \rightarrow \samp'})} + \delta \right)  \\
={} &\frac{1}{n} \sum_{i = 1}^{n} \left( e^{\eps} \cdot \ex{\samp', \sample, \monitor}{\left.  \query(\samp') \; \right| \; \query = \monitor(\sample)} + \delta \right) \tag{the pairs $(\samp_{i}, \sample_{i \rightarrow \samp'})$ and $(\samp',\sample)$ are identially distributed} \\
={} &e^{\eps} \cdot \ex{\samp', \sample, \monitor}{\left. \query(\samp') \; \right| \; \query = \monitor(\sample)} + \delta
\\={}& e^{\eps} \cdot \ex{\sample , \monitor}{\left. \query(\dist) \; \right| \; \query = \monitor(\sample)} + \delta
\end{align*}

An identical argument shows that 
$$
\ex{\sample, \monitor}{\query(\sample) \mid \query = \monitor(\sample)}
\geq e^{-\eps} \cdot \left( \ex{\sample, \monitor}{\left. \query(\dist) \; \right| \; \query = \monitor(\sample)} - \delta \right).
$$

Therefore, using the fact that $|q(\dist)| \leq 1$ for any statistical query $q$ and distribution $\dist$, we have
$$
\left|\ex{\sample,\monitor}{\query(\dist) \mid \query = \monitor(\sample)} - \ex{\sample,\monitor}{\query(\sample) \mid \query = \monitor(\sample)} \right| \leq e^{\eps} - 1 + \delta,
$$
as desired.
\end{proof}
\fi

\ifnum\short=0
\subsection{Warmup: A Multi-Sample De-Correlated Expectation Lemma for SQs} \label{sec:SQCE} 
\else
\subsection{A Multi-Sample De-Correlated Expectation Lemma}
\fi
\ifnum\short=0
As a second warmup, in this section we give a simpler version of our main lemma for the case of statistical queries and \emph{multiple} samples.  
That is, we consider a setting where there are many subsamples available to the algorithm. The multi-sample de-correlated expectation lemma says that a max-KL stable algorithm cannot take a collection of samples $\sample_1,\dots,\sample_T$ and output a pair $(\query, t)$ such that $\query(\dist)$ and $\query(\sample_t)$ differ significantly in expectation. 
\else
In this section we prove our ``multi-sample de-correlated expectation lemma.''  Consider a setting where there are many subsamples available to the algorithm. The multi-sample de-correlated expectation lemma says that a max-KL stable algorithm cannot take a collection of samples $\sample_1,\dots,\sample_T$ and output a pair $(\query, t)$ such that $\query(\dist)$ and $\query(\sample_t)$ differ significantly in expectation.  This lemma is formalized as follows.

\begin{lem}[Main Technical Lemma] \label{lem:MKLCondExp}
Let $\querysetDelta$ be the class of $\Delta$-sensitive queries on $\univ$.  Let $\monitor \from (\univ^{n})^{\trials} \to \querysetDelta \times [\trials]$ be $(\eps,\delta)$-max-KL stable. Let $\dist$ be a distribution on $\univ$ and let $\samples = (\sample_1,\dots,\sample_T) \getsr (\dist^n)^\trials$. Then 
$$
\left|\ex{\samples,\monitor}{\query(\dist) \mid (\query,\trial) = \monitor(\samples)} - \ex{\samples,\monitor}{\query(\sample_{\trial}) \mid (\query,\trial) = \monitor(\samples)} \right| \leq 2 (e^{\eps} - 1 + \trials \delta) \Delta n.
$$
\end{lem}

To keep the presentation simple, we will defer the complete proof of this lemma to the full version, and instead prove a weaker version that only consider statistical queries.
\fi
\begin{lem} \label{lem:MKLCondExpSQ}
Let $\monitor \from (\univ^{n})^{\trials} \to \queryset \times [\trials]$ be $(\eps,\delta)$-max-KL stable where $\queryset$ is the class statistical queries $\query : \univ \to [0,1]$. Let $\dist$ be a distribution on $\univ$ and let $\samples = (\sample_1,\dots,\sample_\trials) \getsr (\dist^n)^\trials$. Then
$$
\left|\ex{\samples,\monitor}{\query(\dist) \mid (\query,\trial) = \monitor(\samples)} - \ex{\samples,\monitor}{\query(\sample_{\trial}) \mid (\query,\trial) = \monitor(\samples)} \right| \leq e^{\eps} - 1 + \trials \delta.
$$
\end{lem}

\begin{proof}[Proof of Lemma~\ref{lem:MKLCondExpSQ}]
Before giving the proof, we set up some notation.  Let $\samples = (\sample_1,\dots,\sample_{\trials})$ be a set of $T$ samples where each sample $\sample_\trial = (\samp_{\trial,1},\dots,\samp_{\trial, n})$.  For a single element $\samp' \in \univ$, and a pair of indices $(m, i) \in [\trials] \times [n]$, we use $\samples_{(m,i) \rightarrow \samp'}$ to denote the new set of $T$ samples where the $i$-th element of the $m$-th sample of $\samples$ has been replaced by the element $\samp'$.

We can now calculate
\ifnum\short=0
\begin{align*}
&\ex{\samples, \monitor}{\query(\sample_{\trial}) \mid (\query,\trial) = \monitor(\samples)} \\
={} &\sum_{m = 1}^{\trials}  \ex{\samples, \monitor}{\left. \mathbf{1}_{\{ t = m \}} \cdot \query(\sample_{m}) \; \right| \; (\query,\trial) = \monitor(\samples)} \\
={} &\frac{1}{n} \sum_{i = 1}^{n} \sum_{m = 1}^{\trials} \ex{\samples, \monitor}{\left.  \mathbf{1}_{\{ t = m \}} \cdot \query(\samp_{m, i}) \; \right| \; (\query,\trial) = \monitor(\samples)} \\
={} &\frac{1}{n} \sum_{i = 1}^{n} \sum_{m = 1}^{\trials} \int_0^1 \pr{\samples, \monitor}{\left.  \mathbf{1}_{\{ t = m \}} \cdot \query(\samp_{m, i}) \; \geq z\right| \; (\query,\trial) = \monitor(\samples)} \mathrm{d}z\\
%\end{align*}
\intertext{Now we can apply $(\eps,\delta)$-max-KL stability.}
%\begin{align*}
\leq{} &\frac{1}{n} \sum_{i = 1}^{n} \sum_{m = 1}^{\trials} \left(\int_0^1 e^\eps \pr{\samples, \monitor}{\left.  \mathbf{1}_{\{ t = m \}} \cdot \query(\samp_{m, i}) \; \geq z\right| \; (\query,\trial) ={} \monitor(\samples_{(m,i) \rightarrow \samp'})} +\delta \right)\mathrm{d}z \tag{by $(\eps, \delta)$-max-KL stability}\\
=& \frac{1}{n} \sum_{i = 1}^{n} \sum_{m = 1}^{\trials} \left( e^{\eps} \cdot \ex{\samp', \samples, \monitor}{\left.  \mathbf{1}_{\{ t = m \}} \cdot \query(\samp_{m, i}) \; \right| \; (\query,\trial) = \monitor(\samples_{(m,i) \rightarrow \samp'})} + \delta \right)  \\
={} &\frac{1}{n} \sum_{i = 1}^{n} \sum_{m = 1}^{\trials} \left( e^{\eps} \cdot \ex{\samp', \samples, \monitor}{\left.  \mathbf{1}_{\{ t = m \}} \cdot \query(\samp') \; \right| \; (\query,\trial) = \monitor(\samples)} + \delta \right) \tag{the pairs $(\samp_{m,i}, \samples_{(m,i) \rightarrow \samp'})$ and $(\samp',\samples)$ are identially distributed} \\
={} &e^{\eps} \cdot \ex{\samp', \samples, \monitor}{\left. \query(\samp') \; \right| \; (\query,\trial) = \monitor(\samples)} + T \delta
\\={}& e^{\eps} \cdot \ex{\samples, \monitor}{\left. \query(\dist) \; \right| \; (\query, \trial) = \monitor(\samples)} + T \delta
\\\leq{}& \ex{\samples, \monitor}{\left. \query(\dist) \; \right| \; (\query, \trial) = \monitor(\samples)} + e^\eps-1 + T \delta \tag{since $\query(\dist) \in [0,1]$}
\end{align*}
\else
\begin{align*}
&\ex{\samples, \monitor}{\query(\sample_{\trial}) \mid (\query,\trial) = \monitor(\samples)} \\
={} &\frac{1}{n} \sum_{i = 1}^{n} \sum_{m = 1}^{\trials} \ex{\samples, \monitor}{\left.  \mathbf{1}_{\{ t = m \}} \cdot \query(\samp_{m, i}) \; \right| \; (\query,\trial) = \monitor(\samples)} \\
\leq& \frac{1}{n} \sum_{i = 1}^{n} \sum_{m = 1}^{\trials} \left( e^{\eps} \cdot \ex{\samp', \samples, \monitor}{\left.  \mathbf{1}_{\{ t = m \}} \cdot \query(\samp_{m, i}) \; \right| \; (\query,\trial) = \monitor(\samples_{(m,i) \rightarrow \samp'})} + \delta \right)  \tag{by $(\eps, \delta)$-max-KL stability and $0 \leq \mathbf{1}_{\{t = m\}} \cdot q(x_{m,i}) \leq 1$}\\
={} &\frac{1}{n} \sum_{i = 1}^{n} \sum_{m = 1}^{\trials} \left( e^{\eps} \cdot \ex{\samp', \samples, \monitor}{\left.  \mathbf{1}_{\{ t = m \}} \cdot \query(\samp') \; \right| \; (\query,\trial) = \monitor(\samples)} + \delta \right) \tag{the pairs $(\samp_{m,i}, \samples_{(m,i) \rightarrow \samp'})$ and $(\samp',\samples)$ are identially distributed} \\
={} &e^{\eps} \cdot \ex{\samp', \samples, \monitor}{\left. \query(\samp') \; \right| \; (\query,\trial) = \monitor(\samples)} + T \delta
={} e^{\eps} \cdot \ex{\samples, \monitor}{\left. \query(\dist) \; \right| \; (\query, \trial) = \monitor(\samples)} + T \delta \\
\leq{}& \ex{\samples, \monitor}{\left. \query(\dist) \; \right| \; (\query, \trial) = \monitor(\samples)} + e^\eps-1 + T \delta \tag{since $\query(\dist) \in [0,1]$}
\end{align*}
\fi
\ifnum\short =0
An identical argument shows that 
$$
\ex{\samples, \monitor}{\query(\sample_{\trial}) \mid (\query,\trial) = \monitor(\samples)}
\geq \ex{\samples, \monitor}{\left. \query(\dist) \; \right| \; (\query, \trial) = \monitor(\samples)} + (e^{-\eps}-1)- T \delta.
$$
\end{proof}
\else
This gives one half of the lemma, the other half is identical.
\end{proof}
\fi

\ifnum\short=0
\subsection{A Multi-Sample De-Correlated Expectation Lemma} \label{sec:lowsensCE}

Here, we give the most general de-correlated expectation lemma that considers multiple samples and applies to the more general class of low-sensitivity queries.

\begin{lem}[Main Technical Lemma] \label{lem:MKLCondExp}
Let $\monitor \from (\univ^{n})^{\trials} \to \querysetDelta \times [\trials]$ be $(\eps,\delta)$-max-KL stable where $\querysetDelta$ is the class of $\Delta$-sensitive queries $\query : \univ^n \to \mathbb{R}$. Let $\dist$ be a distribution on $\univ$ and let $\samples = (\sample_1,\dots,\sample_T) \getsr (\dist^n)^\trials$. Then 
$$
\left|\ex{\samples,\monitor}{\query(\dist) \mid (\query,\trial) = \monitor(\samples)} - \ex{\samples,\monitor}{\query(\sample_{\trial}) \mid (\query,\trial) = \monitor(\samples)} \right| \leq 2 (e^{\eps} - 1 + \trials \delta) \Delta n.
$$
\end{lem}
We remark that if we use the weaker assumption that $\monitor$ is $(e^{\eps}-1+\delta)$-TV stable, (defined in Section~\ref{sec:other-notions}), then we would obtain the same conclusion but with the weaker bound of $2\trials(e^{\eps} - 1 + \delta)\Delta n.$  The advantage of using the stronger definition of max-KL stability is that we only have to decrease $\delta$ with $\trials$ and not $\eps.$  This advantage is crucial because algorithms satisfying $(\eps, \delta)$-max-KL stability necessarily have a linear dependence on $1/\eps$ but only a polylogarithmic dependence on $1/\delta$ \rcomm{in their sample complexity}. 

\begin{proof}[Proof of Lemma~\ref{lem:MKLCondExp}]
Let $\samples' = (\sample'_1,\dots,\sample'_{\trials}) \getsr (\dist^{n})^{\trials}$ be independent of $\samples$.  Recall that each element $\sample_{\trial}$ of $\samples$ is itself a vector $(\samp_{\trial, 1},\dots,\samp_{\trial, n}),$ and the same is true for each element $\sample'_{\trial}$ of $\samples'.$  We will sometimes refer to the vectors $\sample_{1},\dots,\sample_{\trials}$ as the \emph{subsamples of $\samples$.}

We define a sequence of intermediate samples that allow us to interpolate between $\samples$ and $\samples'$ using a series of neighbouring samples.  Formally, for $\ell \in \{0,1,\dots,n\}$ and $m \in \{0,1,\dots,\trials\}$, define $\samples^{\ell,m} = (\sample^{\ell, m}_{1},\dots,\sample^{\ell, m}_{\trials}) \in (\univ^{n})^{\trials}$ by
$$
\samp^{\ell,m}_{\trial, i} = 
\left\{ \begin{array}{cl}
	\samp_{\trial, i} & (\trial > m) \textrm{ or } (\trial = m \textrm{ and } i > \ell)\\
	\samp'_{\trial, i} & (\trial < m) \textrm{ or } (\trial = m \textrm{ and } i \leq \ell) \\
\end{array} \right.
$$
By construction we have $\samples^{0,1}=\samples^{n,0}=\samples$ and $\samples^{n,T} = \samples'.$ Also $\samples^{0, m} = \samples^{n, m-1}$ for $m \in [\trials]$.  Moreover, pairs $(\samples^{\ell, \trial}, \samples^{\ell-1, \trial})$ are neighboring in the sense that there is a single subsample, $\sample_{\trial}$ such that $\sample^{\ell, \trial}_{\trial}$ and $\sample^{\ell-1,\trials}_{\trial}$ are neighbors and for every $\trial' \neq \trial$, $\sample^{\ell, \trial}_{\trial'}=\sample^{\ell -1, \trial}_{\trial'}$. 

For $\ell \in [n]$ and $m \in [\trials]$, define a randomized function $B^{\ell, m} : (\univ^{n})^{\trials} \times (\univ^{n})^{\trials}\to \R$ by
$$
B^{\ell,m}(\samples,\mathbf{Z}) = 
\left\{\begin{array}{cl}
	\query(\mathbf{z}_{\trial}) - \query(\mathbf{z}_{\trial, -\ell}) + \Delta & \trial = m\\
	0 & \trial \neq m
\end{array} \right.
\textrm{\; where $(\query, \trial) = \monitor(\samples),$}
$$
where $\mathbf{z}_{\trial, -\ell}$ is the $\trial$-th subsample of $\mathbf{Z}$ with its $\ell$-th element replaced by some arbitrary fixed element of $\univ.$

We can now expand $\left|\ex{\samples, \monitor}{\query(\dist) - \query(\sample_{\trial}) \mid (\query,\trial) = \monitor(\samples)} \right|$ in terms of these intermediate samples and the functions $B^{\ell,m}$:
\begin{align*}
&\left|\ex{\samples, \monitor}{\query(\dist) - \query(\sample_{\trial}) \mid (\query,\trial) = \monitor(\samples)} \right| \\
={} &\left|\ex{\samples, \samples', \monitor}{\query(\sample'_{\trial}) - \query(\sample_{\trial})\mid (\query,\trial) = \monitor(\samples)} \right| \\
={} &\left| \sum_{\ell \in [n]} \sum_{m \in [\trials]} \ex{\samples,\samples',\monitor}{\left. \query(\sample^{\ell,m}_{\trial})- \query(\sample^{\ell-1,m}_{\trial}) \; \right| \; (\query,\trial) = \monitor(\samples)} \right|\\	
\leq{} &\sum_{\ell \in [n]} \left|\sum_{m \in [\trials]} \ex{\samples,\samples',\monitor}{\left. \query(\sample^{\ell,m}_{\trial})- \query(\sample^{\ell-1,m}_{\trial}) \; \right| \; (\query,\trial) = \monitor(\samples)} \right|\\
={} &\sum_{\ell \in [n]} \left|\sum_{m \in [\trials]} \ex{\samples,\samples',\monitor}{\left. \left(\query(\sample^{\ell,m}_{\trial}) - \query(\sample^{\ell, m}_{\trial, -\ell}) + \Delta\right) - \left(\query(\sample^{\ell-1,m}_{\trial}) - \query(\sample^{\ell-1,m}_{\trial,-\ell}) + \Delta\right) \; \right| \; (\query,\trial) = \monitor(\samples)} \right| \tag{By construction, $\sample^{\ell, m}_{\trial, -\ell} = \sample^{\ell-1,m}_{\trial, -\ell}$}\\
={} &\sum_{\ell \in [n]} \left|\sum_{m \in [\trials]} \ex{\samples, \samples', \monitor}{B^{\ell,m}(\samples,\samples^{\ell,m}) - B^{\ell,m}(\samples,\samples^{\ell-1,m})} \right| \tag{Definition of $B^{\ell, m}.$}
\end{align*}
Thus, it suffices to show that $ \left|\sum_{m \in [\trials]} \ex{}{B^{\ell,m}(\samples,\samples^{\ell,m}) - B^{\ell,m}(\samples,\samples^{\ell-1,m})} \right| \leq 2 (e^{\eps} - 1 + T\delta)\Delta$ for all $\ell \in [n]$. To this end, we make a few observations.
\begin{enumerate}
\item Since $\query$ is $\Delta$-sensitive, for every $\ell,m, \samples, \mathbf{Z}$, we have $0 \leq B^{\ell,m}(\samples,\mathbf{Z}) \leq 2\Delta$.  Moreover, since $B^{\ell, m}(\samples, \mathbf{Z}) = 0$ whenever $\monitor(\samples)$ outputs $(\query, \trial)$ with $\trial \neq m$, we have $\sum_{m \in [\trials]} \ex{}{B^{\ell,m}(\sample,\sample^{\ell,m})} \leq 2\Delta$.%

\item By construction, $B^{\ell, m}(\samples, \mathbf{Z})$ is $(\eps, \delta)$-max-KL stable as a function of its first parameter $\samples.$  Stability follows by the post-processing lemma (Lemma~\ref{lem:postprocessing}) since $B^{\ell, m}$ is a post-processing of the output of $\monitor(\samples)$, which is assumed to be $(\eps, \delta)$-max-KL stable.

\item Lastly, observe that the random variables $\samples^{\ell, m}$ are identically distributed (although they are not independent).  Namely, each one consists of $n \trials$ independent samples from $\dist.$  Moreover, for every $\ell$ and $m$, the pair $(\samples^{\ell, m}, \samples)$ has the same distribution as $(\samples, \samples^{\ell, m}).$  Specifically, the first component is $n \trials$ independent samples from $\dist$ and the second component is equal to the first component with a subset of the entries replaced by fresh independent samples from $\dist.$\end{enumerate}

Consider the random variables $B^{\ell,m}(\samples,\samples^{\ell,m})$ and $B^{\ell,m}(\samples,\samples^{\ell-1,m})$ for some $\ell \in [n]$ and $m \in [\trials]$.  Using observations 2 and 3, we have
$$
B^{\ell,m}(\samples,\samples^{\ell,m}) \sim B^{\ell,m}(\samples^{\ell,m},\samples) \sim_{(\eps,\delta)} B^{\ell,m}(\samples^{\ell-1,m},\samples) \sim B^{\ell,m}(\samples,\samples^{\ell-1,m}),
$$
where $\sim$ denotes having the same distribution and $\sim_{(\varepsilon,\delta)}$ denotes having $(\varepsilon,\delta)$-max-KL close distributions.\footnote{In the spirit of $(\eps, \delta)$-max-KL stability, we say that distributions $A$ and $B$ over $\cR$ are $(\eps, \delta)$-max-KL close if for every $R \subseteq \cR$, $\pr{}{A \in R} \leq e^{\eps} \cdot \pr{}{B \in R} + \delta$ and vice versa.}  Thus $B^{\ell,m}(\samples,\samples^{\ell-1,m})$ and $B^{\ell,m}(\samples,\samples^{\ell,m})$ are $(\eps,\delta)$-max-KL close.

Now we can calculate
\begin{align*}
\ex{\samples, \samples', \monitor}{B^{\ell,m}(\samples,\samples^{\ell-1,m})} 
={} &\int_0^{2\Delta} \pr{\samples, \samples', \monitor}{B^{\ell,m}(\samples,\samples^{\ell-1,m}) \geq z} \mathrm{d}z\\
\leq{} &\int_0^{2\Delta} \left(e^\eps \cdot \pr{\samples, \samples', \monitor}{B^{\ell,m}(\samples,\samples^{\ell,m}) \geq z} + \delta\right) \mathrm{d}z\\
={} &e^\eps \cdot \int_0^{2\Delta} \pr{\samples, \samples', \monitor}{B^{\ell,m}(\samples,\samples^{\ell,m}) \geq z} \mathrm{d}z + 2\delta \Delta\\
={} &e^\eps \cdot \ex{\samples, \samples', \monitor}{B^{\ell,m}(\samples,\samples^{\ell,m})} + 2\delta \Delta.
\end{align*}
Thus we have 
\begin{align*}
\sum_{m \in [\trials]} \ex{\samples, \samples', \monitor}{B^{\ell,m}(\samples,\samples^{\ell-1,m})}
\leq{} &e^\eps \cdot \left(\sum_{m \in [\trials]} \ex{\samples, \samples, \monitor}{B^{\ell, m}(\samples,\samples^{\ell,m})} \right)+ 2 \trials\delta \Delta \\
\leq{} &\sum_{m \in [\trials]} \ex{\samples, \samples', \monitor}{B^{\ell,m}(\samples,\samples^{\ell,m})}  + 2(e^\eps-1)\Delta + 2\Delta \trials\delta.
\end{align*}
Thus we have the desired upper bound on the expectation of $\sum_{m \in [\trials]} \ex{}{B^{\ell,m}(\samples,\samples^{\ell,m}) - B^{\ell,m}(\samples,\samples^{\ell-1,m})}.$ The corresponding lower bound follows from an analogous argument. This completes the proof.
\end{proof}
\fi

\subsection{From Multi-Sample De-Correlated Expectation to Accuracy}
Now that we have Lemma~\ref{lem:MKLCondExp}, we can prove the following result that max-KL stable mechanisms that are also accurate with respect to their sample are also accurate with respect to the population from which that sample was drawn.
\begin{thm}[Main Transfer Theorem] \label{thm:MKLtoAdaptive}
Let $\queryset$ be a family of $\Delta$-sensitive queries on $\univ$.  Assume that, for some $\alpha, \beta \in (0,.1)$, $\mechanism$ is
\begin{enumerate}
\item $(\eps = \alpha / 64 \Delta n, \delta = \alpha \beta / 32 \Delta n)$-max-KL stable for $k$ adaptively chosen queries from $\queryset$ and
\item $(\alpha' = \alpha / 8, \beta' = \alpha\beta / 16 \Delta n)$-accurate with respect to its sample for $n$ samples from $\univ$ for $\queries$ adaptively chosen queries from $\queryset$.
\end{enumerate}
Then $\mechanism$ is $(\alpha,\beta)$-accurate with respect to the population for $\queries$ adaptively chosen queries from $\queryset$ given $n$ samples from $\univ$.
\end{thm}

The key step in the proof is to define a monitoring algorithm that takes $T$ separate samples $\samples = (\sample_1,\dots,\sample_T)$ and for each sample $\sample_t$, simulates an independent interaction between $\mechanism(\sample_t)$ and $\adv$.  This monitoring algorithm then outputs the query with the largest error across all of the queries and interactions ($kT$ queries in total).  Since changing one input to $\samples$ only affects one of the simulations, the monitoring algorithm will be stable so long as $\mechanism$ is stable, without any loss in the stability parameter.  On the other hand, if $\mechanism$ has even a small chance $\beta$ of answering a query with large error, then if we simulate $T \approx 1/\beta$ independent interactions, there is a constant probability that at least one of the simulations results in a query with large error.  Thus, the monitor will be a stable algorithm that outputs a query with large error \emph{in expectation}.  By the multi-sample de-correlated expectation lemma, such a monitor is impossible, which implies that $\mechanism$ has probability $\leq \beta$ of answering any query with large error.

\begin{proof}[Proof of Theorem~\ref{thm:MKLtoAdaptive}]
Let $\mechanism$ be an interactive mechanism. Let $\adv$ be an analyst and let $\dist$ be the distribution chosen by $\adv$.  We define the following monitoring algorithm.
\begin{figure}[h]
\begin{framed}
\begin{center}$\monitor(\samples) = \monitor_{\dist}[\mechanism, \adv](\samples):$\end{center}
\begin{algorithmic}
\INDSTATE[0]{\textbf{Input:} $\samples = (\sample_1,\dots,\sample_T) \in (\univ^n)^{\trials}$}
\INDSTATE[1]{For $\trial = 1,\dots,\trials:$}
\INDSTATE[2]{Simulate $\mechanism(\sample_{\trial})$ and $\adv$ interacting, let $\query_{\trial,1},\dots,\query_{\trial,k} \in \queryset$ be the queries of $\adv$ and let}
\INDSTATE[2]{$a_{\trial,1},\dots,a_{\trial,k} \in \R$ be the corresponding answers of $\mechanism.$}
\INDSTATE[1]{Let $$(j^*, \trial^*) = \argmax_{j \in [\queries],\, \trial \in [\trials]} \left|\errp{\dist}{\query_{\trial,j}}{a_{\trial,j}}\right|.$$}
\INDSTATE[1]{If $a_{\trial^*, j^*} - \query_{\trial^*, j^*}(\dist) \geq 0$, let $\query^* = \query_{\trial^*,j^*},$ otherwise let $\query^* = -\query_{\trial^*, j^*}.$ ($\querysetDelta$ is closed under negation.)}
\INDSTATE[0]{\textbf{Output:} $(\query^*, \trial^*).$}
\end{algorithmic}
\end{framed}
%\vspace{-5mm}
%\caption{Foo}
\end{figure}

If $\mechanism$ is stable then so is $\monitor,$ and this fact follows easily from the post-processing lemma (Lemma~\ref{lem:postprocessing}):
\begin{claim} \label{clm:MKLstablemonitor}
For every $\eps, \delta \geq 0$, if the mechanism $\mechanism$ is $(\eps,\delta)$-max-KL stable for $\queries$ adaptively chosen queries from $\queryset,$ then for every $\dist$ and $\adv$, the monitor $\monitor_{\dist,\queries,\queryset}[\mechanism, \adv]$ is $(\eps, \delta)$-max-KL stable.
\end{claim}
\ifnum\short=0
\begin{proof}
If $\mechanism$ is $(\eps, \delta)$-max-KL stable for $k$ adaptively chosen queries from $\queryset$ then for every analyst $\adv$ who asks $\queries$ queries from $\queryset$, and every $\trial$ the algorithm $\monitor'(\sample_{\trial})$ that simulates the interaction between $\mechanism(\sample_{\trial})$ and $\adv$ and outputs the resulting query-answer pairs is $(\eps, \delta)$-max-KL stable.  From this, it follows that the algorithm $\monitor'(\samples)$ that simulates the interactions between $\mechanism(\sample_{\trial})$ and $\adv$ for every $t = 1,\dots,\trials$ and outputs the resulting query-answer pairs is $(\eps, \delta)$-max-KL stable.  To see this, observe that if $\samples, \samples'$ differ only on one subsample $\sample_{\trial}$, then for every $\trial' \neq \trial$, $\sample_{\trial'} = \sample'_{\trial'}$ and thus the query-answer pairs corresponding to subsample $\trial'$ are identically distributed regardless of whether we use $\samples$ or $\samples'$ as input to $\monitor.$

Observe that the algorithm $\monitor$ defined above is simply a post-processing of these $\queries \trials$ query-answer pairs.  That is, $(\query^*, \trial^*)$ depends only on $\{ (\query_{\trial, j}, a_{\trial, j} \}_{\trial \in [\trials], j \in [k]}$ and $\dist$, and not on $\samples.$  Thus, by Lemma~\ref{lem:postprocessing}, $\monitor$ is $(\eps,\delta)$-max-KL stable.
\end{proof}
\fi

We will use the $\monitor$ with $T = \left\lfloor 1/\beta \right\rfloor.$
In light of Claim~\ref{clm:MKLstablemonitor} and our assumption that $\mechanism$ is $(\eps, \delta)$-max-KL stable, we can apply Lemma~\ref{lem:MKLCondExp} to obtain
\begin{equation}\label{eq:MKLMonitorCondExp}
\left| \ex{\samples, \monitor}{\query^*(\dist) - \query^*(\sample_{\trial^*}) \mid (\query^*, \trial^*) = \monitor(\samples)} \right| \leq 2\left(e^{\alpha / 64 \Delta n} - 1 + \trials\left( \frac{\alpha \beta}{32 \Delta n} \right)\right)\Delta n \leq \alpha / 8.
\end{equation}
To complete the proof, we show that if $\mechanism$ is not $(\alpha, \beta)$-accurate with respect to the population $\dist$, then~\eqref{eq:MKLMonitorCondExp} cannot hold.  
To do so, we need the following natural claim about the output of the monitor.
\begin{claim}\label{clm:amplify}
$\pr{\samples,\monitor}{\query^*(\dist) - a_{\query^*} > \alpha} > 1-(1-\beta)^\trials,$ and $\query^*(\dist) - a_{\query^*} \geq 0,$
where $a_{\query^*}$ is the answer to $\query^*$ produced during the simulation.
\end{claim}

\ifnum\short=1
The proof is relatively straightforward using the fact that for each subsample $\sample_t$, $\mechanism$ outputs an answer with error $\alpha$ with probability at least $\beta$, and the executions of $\mechanism$ on different subsamples are independent.  Also note that the final step of the monitor ensures that the difference is always positive.
\else
\begin{proof}
Since $\mechanism$ fails to be $(\alpha, \beta)$-accurate, for every $\trial \in [\trials]$,
\begin{equation} \label{eq:MKLaccuracyfails}
\pr{\sample_{\trial},\mechanism}{\max_{j \in [\queries]} \left|\query_{\trial, j}(\dist) - a_{\trial, j} \right|>\alpha}>\beta.
\end{equation}
We obtain the claim from~\eqref{eq:MKLaccuracyfails} by using the fact that the $\trials$ sets of query-answer pairs corresponding to different subsamples $\sample_1,\dots,\sample_{\trials}$ are independent.  That is, the random variables $\max_{j \in [\queries]} \left|\query_{\trial, j}(\dist) - a_{\trial, j} \right|$ indexed by $\trial \in [\trials]$ are independent. Since $\query^*(\dist) - a_{\query^*}$ is simply the maximum of these independent random variables, the first part of the claim follows. Also, by construction, $\monitor$ ensures that
\begin{equation} \label{eq:MKLnonnegativeerror}
\query^*(\dist) - a_{\query^*} \geq 0.
\end{equation}
\end{proof}
\fi

\begin{claim}\label{clm:TVContradiction-hp}
If $\mechanism$ is $(\alpha', \beta')$-accurate for the sample but not $(\alpha, \beta)$-accurate for the population, then
$$
\left| \ex{\samples, \monitor}{\query^*(\dist) - \query^*(\sample_{\trial^*}) \mid (\query^*, \trial^*) = \monitor(\samples)} \right| \geq \alpha / 4.
$$
\end{claim} 
\begin{proof}
Now we can calculate
\ifnum\short=0
\begin{align}
&\left| \ex{\samples, \monitor}{\query^*(\dist) - \query^*(\sample_{\trial^*}) \mid (\query^*, \trial^*) = \monitor(\samples)} \right| \notag \\
={} &\left| \ex{\samples, \monitor}{\query^*(\dist) - a_{\query^*} \mid (\query^*, \trial^*) = \monitor(\samples)} +   \ex{\samples, \monitor}{a_{\query^*} - \query^*(\sample_{\trial^*}) \mid (\query^*, \trial^*) = \monitor(\samples)} \right| \notag \\
\geq{} &\left| \ex{\samples, \monitor}{\query^*(\dist) - a_{\query^*} \mid (\query^*, \trial^*) = \monitor(\sample)} \right| - \left|  \ex{\samples, \monitor}{a_{\query^*} - \query^*(\sample_{\trial^*}) \mid \query^* = \monitor(\sample)} \right| \notag \\
\geq{} &\alpha (1 - (1-\beta)^{\trials}) - \left|  \ex{\samples, \monitor}{a_{\query^*} - \query^*(\sample_{\trial^*}) \mid (\query^*, \trial^*) = \monitor(\samples)} \right| \tag{Claim \ref{clm:amplify}}\\%\tag{Combining~\eqref{eq:MKLamplifyerror} and~\eqref{eq:MKLnonnegativeerror}} \\
\geq{} &\alpha (1 - (1-\beta)^{\trials})  - \left(\alpha/8 + 2\trials \left(\frac{\alpha \beta}{16\Delta n} \right) \Delta n\right) \label{eq:MKLSampAccurate} \\
\geq{} &\alpha/2 - \left(\alpha/8+ \alpha/8 \right) ={} \alpha/4 \tag{$\trials = \lfloor 1/\beta \rfloor.$}
\end{align}
Line~\eqref{eq:MKLSampAccurate} follows from two observations.
\else
\begin{align}
&\left| \ex{\samples, \monitor}{\query^*(\dist) - \query^*(\sample_{\trial^*}) \mid (\query^*, \trial^*) = \monitor(\samples)} \right| \notag \\
\geq{} &\left| \ex{\samples, \monitor}{\query^*(\dist) - a_{\query^*} \mid (\query^*, \trial^*) = \monitor(\sample)} \right| - \left|  \ex{\samples, \monitor}{a_{\query^*} - \query^*(\sample_{\trial^*}) \mid \query^* = \monitor(\sample)} \right| \notag \\
\geq{} &\alpha (1 - (1-\beta)^{\trials}) - \left|  \ex{\samples, \monitor}{a_{\query^*} - \query^*(\sample_{\trial^*}) \mid (\query^*, \trial^*) = \monitor(\samples)} \right| \tag{Claim~\ref{clm:amplify}} \\
\geq{} &\alpha (1 - (1-\beta)^{\trials})  - \left(\alpha/8 + 2\trials \left(\frac{\alpha \beta}{16\Delta n} \right) \Delta n\right) 
\geq{} \alpha/2 - \left(\alpha/8+ \alpha/8 \right) ={} \alpha/4 \tag{$\trials = \lfloor 1/\beta \rfloor.$}
\end{align}
The second to last inequality follows from two observations.
\fi
First, since $\mechanism$ is assumed to be $(\alpha / 8, \alpha \beta  / 16 \Delta n)$-accurate for one sample, by a union bound, it is simultaneously $(\alpha / 8, \trials (\alpha \beta  / 16 \Delta n))$-accurate for all of the $\trials$ samples.  Thus, we have $a_{\query^*} - \query^*(\sample_{\trial^*}) \leq \alpha'$ except with probability at most $\trials (\alpha \beta / 16 \Delta n).$  Second, since $\query^*$ is a $\Delta$-sensitive query, we always have $a_{\query^*} - \query^*(\sample_{\trial^*}) \leq 2\Delta n.$\footnote{Without loss of generality, the answers of $\mechanism$ can be truncated to an interval of width $2\Delta n$ that contains the correct answer $\query^*(\sample_{\trial^*}).$ Doing so will ensure $| a_{\query^*} - \query^*(\sample_{\trial^*}) | \leq 2\Delta n.$}
\end{proof}
Thus, if $\mechanism$ is not $(\alpha, \beta)$-accurate for the population, we will obtain a contradiction to~\eqref{eq:MKLMonitorCondExp}.  This completes the proof.
\end{proof}

\ifnum\short=0 %%%Search for OTHERNOTIONS to find matching \fi

\section{Other Notions of Stability and Accuracy on Average}
\label{sec:othernotions-bounds}

Definition \ref{def:KLstable} gives one notion of stability, namely max-KL stability. However, this is by no means the only way to formalise stability for our purposes. In this section we consider other notions of stability and the advantages they have.

\subsection{Other Notions of Algorithmic Stability}\label{sec:other-notions}
We will define here other notions of algorithmic stability, and in Section~\ref{subsec:avg-accuracy}, we will show that such notions can provide expected guarantees for generalization error which can be used to achieve accuracy on average.
\begin{definition}[TV-Stability]
\label{def:TVstable}
Let $\alg \from \univ^n \to \cR$ be a randomized algorithm.  We say that $\alg$ is \emph{$\eps$-TV stable} if for every pair of samples that differ on exactly one element,
$$
\tvd(\alg(\sample), \alg(\sample')) = \sup_{R \subseteq \cR} \left| \pr{}{\alg(\sample) \in R} - \pr{}{\alg(\sample') \in R} \right| \leq \eps.
$$
\end{definition}

\begin{definition}[KL-Stability]
\label{def:KLstable}
Let $\alg \from \univ^n \to \cR$ be a randomized algorithm.  We say that $\alg$ is \emph{$\eps$-KL-stable} if for every pair of samples $\sample, \sample'$ that differ on exactly one element, 
$$
\ex{r \getsr \alg(\sample)}{\log\left( \frac{\pr{}{\alg(\sample) = r}}{\pr{}{\alg(\sample') = r}} \right)} \leq 2\eps^2
$$
\end{definition}

The post-processing property of Max-KL stability (Lemma~\ref{lem:postprocessing} in Section~\ref{subsec:max-kl}) also applies to the two stability notions above.
\begin{lemma}[Stability Notions Preserved under Post-Processing] \label{lem:postprocessing_gen}
Let $\alg \from \univ^n \to \cR$ and $f \from \cR \to \cR'$ be a pair of randomized algorithms.  If $\alg$ is \{$\eps$-TV, $\eps$-KL, $(\eps, \delta)$-max-KL\}-stable then the algorithm $f(\alg(\sample))$ is \{$\eps$-TV, $\eps$-KL, $(\eps, \delta)$-max-KL\}-stable.
\end{lemma}

\paragraph{Relationships Between Stability Notions}
$\eps$-KL stability implies $\eps$-TV stability by Pinsker's inequality.  The relationship between max-KL stability defined in Section~\ref{subsec:max-kl} and the above notions is more subtle.  When $\eps \leq 1$, $(\eps, 0)$-max-KL stability implies $\eps$-KL stability and thus also $\eps$-TV stability. When $\eps \leq 1$ and $\delta > 0$, $(\eps, \delta)$-max-KL stability implies $(2\eps + \delta)$-TV stability.  It also implies that $\mechanism$ is ``close'' to satisfying $2\eps$-KL stability (cf.~\cite{DworkRV10} for more discussion of these notions).

As in Section~\ref{subsec:interact-stability}, we define TV-stability and KL-stability of an interactive mechanism $\mechanism$ through a noninteractive mechanism that simulates the interaction between $\mechanism$ and an adversary $\adv$.  The definition for these notions of stability is precisely analogous to Definition~\ref{def:interact-stability} for max-KL stability.

As with max-KL stability, both notions above allow for \emph{adaptive composition.}  In fact, $\eps$-TV stability composes linearly---a mechanism that is $\eps$-TV stable for $1$ query is $\eps k$-stable for $k$ queries.  The advantage of the stronger notions of KL and max-KL stability is that they have a stronger composition.  A mechanism that is $\eps$-KL stable for $1$ query is $(\eps \sqrt{k})$-stable for $k$ queries.

\subsection{From TV Stability to Accuracy on Average}\label{subsec:avg-accuracy}
In this section we show that TV stable algorithms guarantee a weaker notion of accuracy on average for adaptively chosen queries.

\subsection{Accuracy on Average}
In Section~\ref{sec:mechanisms} we defined accurate mechanisms to be those that answer accurately (either with respect to the population or the sample) with probability close to $1$.  In this section we define a relaxed notion of accuracy that only requires low error in expectation over the coins of $\mechanism$ and $\adv$.

\begin{definition}[Average Accuracy] \label{def:accuratemechanism-avg}
A mechanism $\mechanism$ is \emph{$\alpha$-accurate on average with respect to the population for $\queries$ adaptively chosen queries from $\queryset$ given $n$ samples in $\univ$} if for every adversary $\adv$,
$$
\ex{\accgame_{n, \queries, \queryset}[\mechanism, \adv]}{
\max_{ j \in [\queries]} ~ \left| \errp{\dist}{q_j}{a_j} \right|}\leq \alpha.
$$
\end{definition}

We will also use a definition of accuracy relative to the sample given to the mechanism:

\begin{definition}[Sample Accuracy on Average] \label{def:accuratemechanism-avg-samp}
A mechanism $\mechanism$ is \emph{$\alpha$-accurate on average with respect to samples of size $n$ from $\univ$ for $\queries$ adaptively chosen queries from $\queryset$} if for every adversary $\adv$,
$$
\ex{\sampaccgame_{n, \queries, \queryset}[\mechanism, \adv]}{\max_{ j \in [\queries] }~~~ \left| \errx{\sample}{q_j}{a_j} \right|} \leq \alpha.
$$
\end{definition}

\subsubsection{A De-Correlated Expectation Lemma}
Towards our goal of proving that TV stability implies accuracy on average in the adaptive setting, we first prove a lemma saying that TV stable algorithms cannot output a low-sensitivity query such that the sample has large error for that query.  In the next section we will show how this lemma implies accuracy on average in the adaptive setting.
\begin{lem} \label{lem:TVCondExp}
Let $\monitor \from \univ^{n} \to \querysetDelta$ be an $\eps$-TV stable randomized algorithm.  Recall $\querysetDelta$ is the family of $\Delta$-sensitive queries $\query \from \univ^n \to \mathbb{R}.$  Let $\dist$ be a distribution on $\univ$ and let $\sample \getsr \dist^{n}.$  Then
$$
\left| \ex{\sample, \monitor}{\query(\dist) \mid \query = \monitor(\sample)} - \ex{\sample, \monitor}{\query(\sample) \mid \query = \monitor(\sample)} \right| \leq 2\eps \Delta n.
$$
\end{lem}
\begin{proof}
The proof proceeds via a sequence of intermediate samples. Let $\sample' \getsr \dist^{n}$ be independent of $\sample$.
For $\ell \in \{0,1,\dots,n\}$, we define $\sample^{\ell} = (\samp_{1}^{\ell},\dots,\samp_{n}^{\ell}) \in \univ^{n}$ by
$$
\samp^{\ell}_{i} = 
\left\{ 
	\begin{array}{cl}
	\samp_{i} &  i > \ell\\
	\samp'_{i} & i \leq \ell
	\end{array} 
\right.
$$
By construction, $\sample^{0} = \sample$ and $\sample^{n} = \sample',$ and intermediate samples $\sample^{\ell}$ interpolate between $\sample$ and $\sample'.$  Moreover, $\sample^{\ell}$ and $\sample^{\ell+1}$ differ in at most one entry, so that we can use the stability condition to relate $\alg(\sample^{\ell})$ and $\alg(\sample^{\ell+1}).$

For every $\ell \in [n]$, we define $B^{\ell} : \univ^{n} \times \univ^{n} \to \R$ by
$$
B^{\ell}(\sample, \mathbf{z}) = 
\query(\mathbf{z}) - \query(\mathbf{z}_{-\ell}) + \Delta,\; \textrm{where $\query = \alg(\sample).$}
$$ 
Here, $\mathbf{z}_{-\ell}$ is $\mathbf{z}$ with the $\ell$-th element replaced by some arbitrary fixed element of $\univ.$

Now we can write
\begin{align*}
&\left| \ex{\sample, \monitor}{\query(\dist) - \query(\sample) \mid \query = \monitor(\sample)} \right| \\
={} &\left|\ex{\sample, \sample', \monitor}{\query(\sample') - \query(\sample) \mid \query = \monitor(\sample)}  \right| \\
={} &\left| \sum_{\ell = 1}^{n} \ex{\sample, \sample', \monitor}{\query(\sample^{\ell}) - \query(\sample^{\ell-1}) \mid \query = \monitor(\sample)} \right| \\
\leq{} &\sum_{\ell = 1}^{n}  \left|\ex{\sample, \sample', \monitor}{\query(\sample^{\ell}) - \query(\sample^{\ell-1}) \mid \query = \monitor(\sample)} \right| \\
={} &\sum_{\ell \in [n]} \left| \ex{\sample, \sample', \monitor}{\left. \left(\query(\sample^{\ell}) - \query(\sample^{\ell}_{-\ell}) + \Delta \right) - \left( \query(\sample^{\ell-1}) - \query(\sample^{\ell-1}_{-\ell}) + \Delta \right) \; \right| \; \query = \monitor(\sample)} \right| \tag{Since $\sample^{\ell}_{-\ell} = \sample^{\ell-1}_{-\ell}$}\\
={} &\sum_{\ell \in [n]} \left| \ex{\sample, \sample', \monitor}{B^{\ell}(\sample,\sample^{\ell}) - B^{\ell}(\sample,\sample^{\ell-1})} \right|. \tag{Definition of $B$}
\end{align*}
Thus, to prove the lemma, it suffices to show that for every $\ell \in [n],$
$$
\left| \ex{\sample, \sample', \monitor}{B^{\ell}(\sample,\sample^{\ell}) - B^{\ell}(\sample,\sample^{\ell-1})} \right| \leq 2\Delta \eps
$$

To complete the proof, we will need a few observations.  First, since $\query$ is $\Delta$-sensitive, for every $\ell, \sample, \mathbf{z}$, we have $0 \leq B^{\ell}(\sample,\mathbf{z}) \leq 2\Delta$.

Second, observe that since $\monitor$ is assumed to be $\eps$-TV stable, by the post-processing lemma (Lemma~\ref{lem:postprocessing}) $B^{\ell}(\sample, \mathbf{z})$ is $\eps$-TV stable with respect to its first parameter $\sample.$

Finally, observe that the random variables $\sample^{0},\dots,\sample^{n}$ are identically distributed (although not independent).  That is, every $\sample^{\ell}$ consists of $n$ independent draws from $\dist.$ Moreover, for every $\ell$, the pairs $(\sample, \sample^{\ell})$ and $(\sample^{\ell}, \sample)$ are identically distributed.  Specifically, the first component is $n$ independent samples from $\dist$ and the second component is equal to the first component with a subset of the entries replaced by new independent samples from $\dist.$

Combining, the second and third observation with the triangle inequality, we have
\begin{align*}
&\tvd\left(B^{\ell}(\sample, \sample^{\ell}), B^{\ell}(\sample, \sample^{\ell-1})\right) \\
\leq{} &\tvd\left(B^{\ell}(\sample, \sample^{\ell}), B^{\ell}(\sample^{\ell}, \sample)\right)
+ \tvd\left(B^{\ell}(\sample^{\ell}, \sample), B^{\ell}(\sample^{\ell-1}, \sample) \right)
+ \tvd\left(B^{\ell}(\sample^{\ell-1}, \sample), B^{\ell}(\sample, \sample^{\ell-1}) \right) \\
\leq{} &0 + \eps + 0 = \eps.
\end{align*}

Using the observations above, for every $\ell \in [n]$ we have 
\begin{align*}
\ex{\sample, \sample', \monitor}{B^{\ell}(\sample,\sample^{\ell}) - B^{\ell}(\sample, \sample^{\ell-1})} 
\leq{} &2 \Delta \cdot \tvd\left(B^{\ell}(\sample, \sample^{\ell}), B^{\ell}(\sample, \sample^{\ell-1})\right)
\leq{} 2\Delta \eps.
\end{align*}
Thus we have the desired upper bound on the expectation of $B^{\ell}(\sample, \sample^{\ell}) - B^{\ell}(\sample, \sample^{\ell-1}).$  The corresponding lower bound follows from an analogous argument.  This completes the proof.
\end{proof}

\subsubsection{From De-Correlated Expectation to Accuracy on Average}
\begin{thm} \label{thm:TVtoAdaptive}
Let $\querysetDelta$ be the family of $\Delta$-sensitive queries on $\univ$.  Assume that $\mechanism$ is
\begin{enumerate}
\item $(\eps = \alpha / 4 \Delta n)$-TV stable for $k$ adaptively chosen queries from $\queryset = \querysetDelta$ and
\item $(\alpha' = \alpha / 2)$-accurate on average with respect to its sample for $n$ samples from $\univ$ for $\queries$ adaptively chosen queries from $\queryset$.
\end{enumerate}
Then $\mechanism$ is $\alpha$-accurate on average with respect to the population for $\queries$ adaptively chosen queries from $\queryset$ given $n$ samples from $\univ$.
\end{thm}
The high level approach of the proof is to apply the Lemma~\ref{lem:TVCondExp} to a ``monitoring algorithm'' that watches the interaction between the mechanism $\mechanism(\sample)$ and the analyst $\adv$ and then outputs the \emph{least accurate} query.  Since $\mechanism(\sample)$ is stable, the de-correlated expectation lemma says that the query output by the monitor will satisfy $\query(\dist) \approx \query(\sample)$ in expectation, this implies that even for the least accurate query in the interaction between $\mechanism(\sample)$ and $\adv$, $\query(\dist) \approx \query(\sample)$ in expectation  Thus, if $\mechanism$ is accurate with respect to the sample $\sample$, it is also accurate with respect to $\dist.$
\begin{proof}[Proof of Theorem~\ref{thm:TVtoAdaptive}]
Let $\mechanism$ be an interactive mechanism and $\adv$ be an analyst that chooses the distribution $\dist$.  We define the following monitoring algorithm.
\begin{figure}[h]
\begin{framed}
\begin{center}$\monitor(\sample) = \monitor_{\dist}[\mechanism, \adv](\sample):$\end{center}
\begin{algorithmic}
\INDSTATE[0]{\textbf{Input:} $\sample \in \univ^n$}
\INDSTATE[1]{Simulate $\mechanism(\sample)$ and $\adv$ interacting, let $\query_1,\dots,\query_k \in \queryset$ be the queries of $\adv$ and let}
\INDSTATE[2]{$a_1,\dots,a_k \in \R$ be the corresponding answers of $\mechanism.$}
\INDSTATE[1]{Let $j = \argmax_{j = 1,\dots,k} \left|\errp{\dist}{\query_j}{a_j}\right|.$}
\INDSTATE[1]{If $a_j-\query_j(\dist) \geq 0$, let $\query^* = \query_j,$ otherwise let $\query^* = -\query_j.$ ($\querysetDelta$ is closed under negation.)}
\INDSTATE[0]{\textbf{Output:} $\query^*.$}
\end{algorithmic}
\end{framed}
%\vspace{-5mm}
%\caption{Foo}
\end{figure}
If $\mechanism$ is stable then so is $\monitor,$ and this fact follows easily from the post-processing lemma (Lemma~\ref{lem:postprocessing}).
\begin{claim} \label{clm:TVstablemonitor}
For every $\eps \geq 0$, if the mechanism $\mechanism$ is $\eps$-TV stable for $k$ adaptively chosen queries from $\queryset,$ then for every $\dist$ and $\adv$, the monitor $\monitor_{\dist}[\mechanism, \adv]$ is $\eps$-TV stable.
\end{claim}
\begin{proof}[Proof of Claim~\ref{clm:TVstablemonitor}]
The assumption that $\mechanism$ is $\eps$-TV stable for $k$ adaptively chosen queries from $\queryset$ means that for every analyst $\adv$ who asks $\queries$ queries from $\queryset$, the algorithm $\monitor'(\sample)$ that simulates the interaction between $\mechanism(\sample)$ and $\adv$ and outputs the resulting query-answer pairs is $\eps$-TV stable.  Observe that the algorithm $\monitor$ defined above is simply a post-processing of these query-answer pairs.  That is, $\query^*$ depends only on $\query_1,a_1,\dots,\query_k, a_k$ and $\dist$, and not on $\sample.$  Thus, by Lemma~\ref{lem:postprocessing}, for every $\dist$ and $\adv$, the monitor $\monitor_{\dist}[\mechanism, \adv]$ is $\eps$-TV stable.
\end{proof}
In light of Claim~\ref{clm:TVstablemonitor} and our assumption that $\mechanism$ is $(\alpha/4\Delta n)$-TV stable, we can apply Lemma~\ref{lem:TVCondExp} to obtain
\begin{equation}\label{eq:TVMonitorCondExp}
\left| \ex{\sample, \monitor}{\query^*(\dist) - \query^*(\sample) \mid \query^* = \monitor(\sample)} \right| \leq 2 \left(\frac{\alpha}{4\Delta n}\right) \Delta n \leq \alpha / 2.
\end{equation}
To complete the proof, we show that if $\mechanism$ is not $\alpha$-accurate on average with respect to the population $\dist$, then~\eqref{eq:TVMonitorCondExp} cannot hold.
\begin{claim}\label{clm:TVContradiction}
If $\mechanism$ is $(\alpha/2)$-accurate for the sample but not $\alpha$-accurate for the population, then
$$
\left| \ex{\sample, \monitor}{\query^*(\dist) - \query^*(\sample) \mid \query^* = \monitor(\sample)} \right| \geq \alpha / 2.
$$
\end{claim}
\begin{proof}[Proof of Claim~\ref{clm:TVContradiction}]
Using our assumptions, we can calculate as follows.
\begin{align}
&\left| \ex{\sample, \monitor}{\query^*(\dist) - \query^*(\sample) \mid \query^* = \monitor(\sample)} \right| \notag \\
={} &\left| \ex{\sample, \monitor}{\query^*(\dist) - a_{\query^*} \mid \query^* = \monitor(\sample)} +   \ex{\sample, \monitor}{a_{\query^*} - \query^*(\sample) \mid \query^* = \monitor(\sample)} \right| \notag \\
\geq{} &\left| \ex{\sample, \monitor}{\query^*(\dist) - a_{\query^*} \mid \query^* = \monitor(\sample)} \right| - \left|  \ex{\sample, \monitor}{a_{\query^*} - \query^*(\sample) \mid \query^* = \monitor(\sample)} \right| \notag \\
>{} &\alpha - \left|  \ex{\sample, \monitor}{a_{\query^*} - \query^*(\sample) \mid \query^* = \monitor(\sample)} \right| \label{eq:TVNotPopAccurate} \\
\geq{} &\alpha - \alpha/2 \label{eq:TVSampAccurate} \\
={} &\alpha / 2 \notag.
\end{align}
Line~\eqref{eq:TVNotPopAccurate} follows from two observations.  First, by construction of $\monitor$, we always have $\query^*(\dist)  - a_{\query^*} \leq 0.$  Second, since $\mechanism$ is assumed not to be $\alpha$-accurate on average for the population, the expected value of $|\query^*(\dist) - a_{\query^*}| > \alpha.$  Since $\monitor$ ensures that $a_{\query^*} - \query^*(\dist) \geq 0,$ we also have that the absolute value of the expectation of $\query^*(\dist) - a_{\query^*}$ is greater than $\alpha.$ Line~\eqref{eq:TVSampAccurate} follows from the assumption that $\mechanism$ is $(\alpha/2)$-accurate on average for the sample.
\end{proof}
Thus, if $\mechanism$ is not $\alpha$-accurate on average for the population, we will obtain a contradiction to~\eqref{eq:TVMonitorCondExp}.  This completes the proof.
\end{proof}

\fi %%%Search for OTHERNOTIONS to find matching \ifnum

\ifnum\short =0
\section{From Low-Sensitivity Queries to Optimization Queries} \label{sec:optimization}
In this section, we extend our results for low-sensitivity queries to the more general family of minimization queries.  To do so, we design a suitable monitoring algorithm for minimization queries.  As in our analysis of low-sensitivity queries, we will have the monitoring algorithm take as input many independent samples and simulate the interaction between $\mechanism$ and $\adv$ on each of those samples.  Thus, if $\mechanism$ has even a small probability of being inaccurate, then with constant probability the monitor will find a minimization query that $\mechanism$ has answered inaccurately.  Previously, we had monitor simply output this query and applied Lemma~\ref{lem:MKLCondExp} to arrive at a contradiction.  However, since Lemma~\ref{lem:MKLCondExp} only applies to algorithms that output a low-sensitivity query, we can't apply it to the monitor that outputs a minimization query.  We address this by having the monitor output the \emph{error function} associated with the loss function and answer it selects, which is a low-sensitivity query.  If we assume that the mechanism is accurate for its sample but not for the population, then the monitor will find a loss function and an answer with low error on the sample but large error on the population.  Thus the error function will be a low-sensitivity query with very different answers on the sample and the population, which is a contradiction.  To summarize, we have the following theorem.
\else
\section{From Low-Sensitivity Queries to Optimization Queries} \label{sec:optimization}
Theorem \ref{thm:MKLtoAdaptive} easily generalizes to minimization queries:
\fi

\begin{thm}[Transfer Theorem for Minimization Queries] \label{thm:MKLtoAdaptiveMinimization}
Let $\queryset = \querysetMIN$ be the family of $\Delta$-sensitive minimization queries on $\univ$.  Assume that, for some $\alpha, \beta \geq 0$, $\mechanism$ is
\begin{enumerate}
\item $(\eps = \alpha / 128 \Delta n, \delta = \alpha \beta / 64 \Delta n)$-max-KL stable for $k$ adaptively chosen queries from $\queryset$ and
\item $(\alpha' = \alpha / 8, \beta' = \alpha\beta / 32 \Delta n)$-accurate with respect to its sample for $n$ samples from $\univ$ for $\queries$ adaptively chosen queries from $\queryset$.
\end{enumerate}
Then $\mechanism$ is $(\alpha,\beta)$-accurate with respect to the population for $\queries$ adaptively chosen queries from $\queryset$ given $n$ samples from $\univ$.
\end{thm}

\ifnum\short =0
The formal proof is nearly identical to that of Theorem~\ref{thm:MKLtoAdaptive}, so we omit the full proof.  Instead, we will simply describe the modified monitoring algorithm.
\begin{figure}[h]
\begin{framed}
\begin{center}$\monitor(\samples) = \monitor_{\dist}[\mechanism, \adv](\samples):$\end{center}
\begin{algorithmic}
\INDSTATE[0]{\textbf{Input:} $\samples = (\sample_1,\dots,\sample_T) \in (\univ^n)^{\trials}$}
\INDSTATE[1]{For $\trial = 1,\dots,\trials:$}
\INDSTATE[2]{Simulate $\mechanism(\sample_{\trial})$ and $\adv$ interacting, let $\loss_{\trial,1},\dots,\loss_{\trial,k} \in \queryset$ be the queries of $\adv$ and let}
\INDSTATE[2]{$\theta_{\trial,1},\dots,\theta_{\trial,k} \in \R$ be the corresponding answers of $\mechanism.$}
\INDSTATE[1]{Let $(\trial^*, j^*)$ be $$(\trial^*, j^*) = \argmax_{j \in [\queries],\, \trial \in [\trials]} \left|\errp{\dist}{\loss_{\trial, j}}{\theta_{\trial,j}}\right|.$$}
\INDSTATE[1]{Let $\query^*(\sample) = \err_{\sample}(\loss_{\trial^*, j^*}, \theta_{\trial^*, j^*})$ (note, by construction, $\query^* \in \queryset_{2\Delta}$, i.e.~$\query^*$ is $2\Delta$-sensitive)}
\INDSTATE[0]{\textbf{Output:} $(\query^*, \trial^*).$}
\end{algorithmic}
\end{framed}
\end{figure}  
\else
The only difference between the proofs of Theorems \ref{thm:MKLtoAdaptive} and \ref{thm:MKLtoAdaptiveMinimization} is the query that the monitor returns: The monitor for Theorem \ref{thm:MKLtoAdaptiveMinimization} selects the loss function $\loss$ and answer $\theta$ from the simulated interactions that maximizes $\errp{\dist}{\loss}{\theta}$ and returns the query $\query^*(\sample) = \err_{\sample}(\loss, \theta)$ (and the index of the interaction that generated $\loss$ and $\theta$). Since $\loss$ is $\Delta$-sensitive in its first argument, $\query^*$ is $2\Delta$-sensitive. 
\fi

\ifnum\short=0 %%% Search for APPLICATIONS to find matching \fi

\section{Applications}

\subsection{Low-Sensitivity and Statistical Queries}
We now plug known stable mechanisms (designed in the context of differential privacy) in to Theorem~\ref{thm:MKLtoAdaptive} to obtain mechanisms that provide strong error guarantees with high probability for both low-sensitivity and statistical queries.

\begin{corollary}[Theorem~\ref{thm:MKLtoAdaptive} and~\cite{DworkMNS06, SteinkeU15}]
There is a mechanism $\mechanism$ that is $(\alpha,\beta)$-accurate with respect to the population for $\queries$ adaptively chosen queries from $\querysetDelta$ where $\Delta = O(1/n)$ given $n$ samples from $\univ$ for
$$
n \geq O\left( \frac{ \sqrt{\queries \cdot \log \log \queries} \cdot \log^{3/2}(1/\alpha \beta)}{\alpha^2} \right) %= \tilde{O}\left( \frac{\sqrt{\queries}\cdot \log(1/\beta)}{\alpha^2} \right).
$$
The mechanism runs in time $\poly(n, \log |\univ|, \log(1/\beta))$ per query.
\end{corollary}

\begin{corollary}[Theorem~\ref{thm:MKLtoAdaptive} and~\cite{RothR10}]
There is a mechanism $\mechanism$ that is $(\alpha,\beta)$-accurate with respect to the population for $\queries$ adaptively chosen queries from $\querysetDelta$ where $\Delta = O(1/n)$ given $n$ samples from $\univ$ for
$$
n = O\left( \frac{ \log |\univ| \cdot \log \queries \cdot \log^{3/2}(1/\alpha \beta)}{\alpha^3} \right) %= \tilde{O}\left( \frac{ \log |\univ| \cdot \log \queries \cdot \log(1/\beta)}{\alpha^3} \right)
$$
The mechanism runs in time $\poly(|\univ|^n)$ per query.  The case where $\Delta$ is not $O(1/n)$ can be handled by rescaling the output of the query.
\end{corollary}

\begin{corollary}[Theorem~\ref{thm:MKLtoAdaptive} and~\cite{HardtR10}]
There is a mechanism $\mechanism$ that is $\alpha$-accurate on average with respect to the population for $\queries$ adaptively chosen queries from $\querysetSQ$ given $n$ samples from $\univ$ for
$$
n = O\left( \frac{ \sqrt{\log |\univ|}\cdot \log \queries  \cdot \log^{3/2}(1/\alpha \beta)}{\alpha^3} \right) %= \tilde{O}\left( \frac{ \sqrt{\log |\univ|} \cdot \log \queries \cdot \log(1/\beta)}{\alpha^3} \right)
$$
The mechanism runs in time $\poly(n, |\univ|)$ per query.
\end{corollary}

\subsection{Optimization Queries}

The results of the Section~\ref{sec:optimization} can be combined with existing differentially private algorithms for minimizing ``empirical risk'' (that is, loss with respect to the sample $x$) to obtain algorithms for answering adaptive sequences of minimization queries. We provide a few specific instantiations here, based on known differentially private mechanisms.

\subsubsection{Minimization Over Arbitrary Finite Sets}

\begin{corollary}[Theorem \ref{thm:MKLtoAdaptiveMinimization} and \cite{McSherryT07}]
Let $\Theta$ be a finite set of size at most $D$. Let $\queryset \subset \querysetMIN$ be the set of sensitivity-$1/n$ loss functions bounded between $0$ and $C$. Then there is a mechanism $\mechanism$ that is $(\alpha,\beta)$-accurate with respect to the population for $k$ adaptively chosen queries from $\querysetMIN$ given
$$
n \geq  O\paren{\frac{\log(DC/\alpha) \cdot \sqrt{k} \cdot \log^{3/2}(1/\alpha \beta) }{\alpha^2}} %= \tilde O \paren{\frac{\log(DC) \cdot \sqrt{k} \cdot \log(1/\beta)}{\alpha^2}}
$$
samples from $\univ$.  The running time of the mechanism is dominated by $O((k+\log(1/\beta))\cdot D)$ evaluations of the loss function.
\end{corollary}

\subsubsection{Convex Minimization}
We state bounds for convex minimization queries for some of the most common parameter regimes in applications.  In the first two corollaries, we consider $1$-Lipschitz\footnote{A loss function $L \from \univ \times \R^d \to \R$ is $1$-Lipschitz if for every $\theta, \theta' \in \R^d$, $x \in \univ$, $|L(\theta, x) - L(\theta', x)| \leq \|\theta - \theta'\|_2$.} loss functions over a bounded domain.

\begin{corollary}[Theorem \ref{thm:MKLtoAdaptiveMinimization} and \cite{BassilyST14}]
  Let $\Theta$ be a closed, convex subset of $\R^d$ set such that $\max_{\theta \in \Theta} \|\theta\|_2 \leq 1$. Let $\queryset \subset \querysetMIN$ be the set of convex $1$-Lipschitz loss functions that are $1/n$-sensitive.  Then there is a mechanism $\mechanism$ that is $(\alpha,\beta)$-accurate with respect to the population for $\queries$ adaptively chosen queries from $\queryset$ given
    $$n = \tilde O\paren{\frac{ \sqrt{d \queries} \cdot  \log^{2}\left(1/\alpha\beta \right)}{\alpha^2}}$$
samples from $\queryset$.
The running time of the mechanism is dominated by $\queries \cdot n^2$ evaluations of the gradient $\nabla\loss$.
\end{corollary}

\begin{corollary}[Theorem \ref{thm:MKLtoAdaptiveMinimization} and \cite{Ullman15}]
  Let $\Theta$ be a closed, convex subset of $\R^d$ set such that $\max_{\theta \in \Theta} \|\theta\|_2 \leq 1$. Let $\queryset \subset \querysetMIN$ be the set of convex $1$-Lipschitz loss functions that are $1/n$-sensitive.  Then there is a mechanism $\mechanism$ that is $(\alpha,\beta)$-accurate with respect to the population for $\queries$ adaptively chosen queries from $\queryset$ given
$$
n = \tilde O\paren{\frac{ \sqrt{\log |\univ|} \cdot (\sqrt{d} + \log \queries) \cdot \log^{3/2}(1/\alpha\beta)}{\alpha^3}}
$$
samples from $\univ$.
The running time of the mechanism is dominated by $\poly(n, |\univ|)$ and $\queries \cdot n^2$ evaluations of the gradient $\nabla\loss$.
\end{corollary}

In the next two corollaries, we consider $1$-strongly convex\footnote{A loss function $L \from \univ \times \R^d \to \R$ is $1$-strongly convex if for every $\theta, \theta' \in \R^d$, $x \in \univ$, $$L(\theta', x) \geq L(\theta, x) + \langle \nabla L(\theta, x), \theta' - \theta \rangle + (1/2) \cdot \|\theta - \theta' \|_2^2,$$ where the (sub)gradient $\nabla L(\theta, x)$ is taken with respect to $\theta$.}, Lipschitz loss functions over a bounded domain.

\begin{corollary}[Theorem \ref{thm:MKLtoAdaptiveMinimization} and \cite{BassilyST14}]
  Let $\Theta$ be a closed, convex subset of $\R^d$ set such that $\max_{\theta \in \Theta} \|\theta\|_2 \leq 1$. Let $\queryset \subset \querysetMIN$ be the set of $1$-strongly convex, $1$-Lipschitz loss functions that are $1/n$-sensitive.  Then there is a mechanism $\mechanism$ that is $(\alpha,\beta)$-accurate with respect to the population for $\queries$ adaptively chosen queries from $\queryset$ given
    $$n = \tilde O\paren{\frac{ \sqrt{d \queries} \cdot \log^{3/2}(1/\alpha\beta)}{\alpha^{3/2}} }$$
samples from $\univ$.
The running time of the mechanism is dominated by $\queries  \cdot n^2$ evaluations of the gradient $\nabla\loss$.
\end{corollary}

\begin{corollary}[Theorem \ref{thm:MKLtoAdaptiveMinimization} and \cite{Ullman15}]
  Let $\Theta$ be a closed, convex subset of $\R^d$ set such that $\max_{\theta \in \Theta} \|\theta\|_2 \leq 1$. Let $\queryset \subset \querysetMIN$ be the set of $1$-strongly convex $1$-Lipschitz loss functions that are $1/n$-sensitive.  Then there is a mechanism $\mechanism$ that is $(\alpha,\beta)$-accurate with respect to the population for $\queries$ adaptively chosen queries from $\queryset$ given
$$
n = \tilde O\paren{ \sqrt{\log |\univ|} \cdot \left(\frac{\sqrt{d}}{\alpha^{5/2}} + \frac{\log \queries}{\alpha^3}\right) \cdot \log^{3/2}(1/\alpha\beta)}
$$
samples from $\univ$.
The running time of the mechanism is dominated by $\poly(n, |\univ|)$ and $\queries \cdot n^2$ evaluations of the gradient $\nabla\loss$.
\end{corollary}
\fi %%% Search for APPLICATIONS to find matching \ifnum

\section{An Alternative Form of Generalization and Tightness of Our Results}\label{sec:opt}
We now provide an alternative form of our generalization bounds.  The following Theorem is more general than Theorem~\ref{thm:MKLtoAdaptive} because it says that \emph{no} max-KL stable procedure that outputs a low-sensitivity can output a query that distinguishes the sample from the population (not just max-KL stable procedures that are accurate for the sample).

First we prove the following technical lemma.
\begin{lem} \label{lem:EMutility}
Let $F$ be a finite set, $f : F \to \mathbb{R}$ a function, and $\eta >0$. Define a random variable $X$ on $F$ by $$\pr{}{X=x} = \frac{e^{\eta f(x)}}{C}, \qquad \text{where} \qquad C= \sum_{x \in F} e^{\eta f(x)}.$$ Then $$\ex{}{f(X)} \geq \max_{x \in F} f(x) - \frac{1}{\eta}\log |F|.$$
\end{lem}
\begin{proof}
We have $$f(x) = \frac{1}{\eta}\left(\log C + \log \pr{}{X=x}\right).$$
Thus
\begin{align*}
\ex{}{f(X)} =& \sum_{x \in F} \pr{}{X=x} f(x)\\
=& \sum_{x \in F} \pr{}{X=x} \frac{1}{\eta}\left( \log C + \log \pr{}{X=x} \right)\\
=& \frac{1}{\eta}\left(\log C - H(X)\right),
\end{align*}
where $H(X)$ is the Shannon entropy of the distribution of $X$ (measured in nats, rather than bits). In particular, $$H(X) \leq \log|\mathrm{support}(X)| = \log|F|,$$ as the uniform distribution maximizes entropy. Moreover, $C \geq \max_{x \in F} e^{\eta f(x)}$, whence $\frac{1}{\eta} \log C \geq \max_{x \in F} f(x)$. The result now follows from these two inequalities.
\end{proof}

\begin{thm} \label{thm:generalization}
Let $\varepsilon \in (0,1/3)$, $\delta \in (0,\varepsilon/4)$, and $n\geq\frac{1}{\eps^2}\log(\frac{4\eps}{\delta})$.
Let $\mechanism \from \univ^{n} \to \querysetDelta$ be $(\eps,\delta)$-max-KL stable where $\querysetDelta$ is the class of $\Delta$-sensitive queries $\query : \univ^n \to \mathbb{R}$. Let $\dist$ be a distribution on $\univ$, let $\sample \getsr \dist^n$, and let $\query\getsr\mechanism(\sample)$. Then 
$$
\pr{\sample,\mechanism}{ \left| \query(\dist) -\query(\sample) \right| \geq 18\eps\Delta n } < \frac{\delta}{\eps}.
$$
\end{thm}

Intuitively, Theorem \ref{thm:generalization} says that ``stability prevents overfitting.'' It says that no stable algorithm can output a low-sensitivity function that distinghishes its input from the population the input was drawn from (i.e. ``overfits'' its sample).

In particular, Theorem \ref{thm:generalization} implies that, if a mechanism $\mechanism$ is stable and outputs $\query$ that ``fits'' its data, then $\query$ also ``fits'' the population. This gives a learning theory perspective on our results.

\begin{proof}
Consider the following monitor algorithm $\monitor$.

\begin{figure}[h]
\begin{framed}
\begin{center}$\monitor(\samples) = \monitor_{\dist}[\mechanism](\samples):$\end{center}
\begin{algorithmic}
\INDSTATE[0]{\textbf{Input:} $\samples = (\sample_1,\dots,\sample_T) \in (\univ^n)^{\trials}$}
\INDSTATE[1]{Set $F=\emptyset$.}
\INDSTATE[1]{For $\trial = 1,\dots,\trials:$}
\INDSTATE[2]{Let $\query_\trial \leftarrow \mechanism(\sample_\trial)$, and set $F=F\cup\{(\query_\trial,\trial), (-\query_\trial,\trial)\}$.}
%\INDSTATE[1]{Use the exponential mechanism with privacy parameter $\epsilon$ and sensitivity parameter $\Delta$ to choose $(\query^*, \trial^*)\in F$ with large $f(\samples,(\query^*, \trial^*))\triangleq\query^*(\sample_{\trial^*})-\query^*(\dist)$.}
\INDSTATE[1]{Sample $(\query^*,\trial^*)$ from $F$ with probability proportional to $\exp\left(\frac{\eps}{\Delta} \left(\query^*(\sample_{\trial^*})-\query^*(\dist)\right)\right)$.}
\INDSTATE[0]{\textbf{Output:} $(\query^*, \trial^*).$}
\end{algorithmic}
\end{framed}
%\vspace{-5mm}
%\caption{Foo}
\end{figure}

We will use the monitor $\monitor$ with $T = \left\lfloor \epsilon/\delta \right\rfloor$.
Observe that $\monitor$ only access its input through $\mechanism$ (which is $(\eps,\delta)$-max-KL-stable) and the exponential mechanism (which is $(\eps,0)$-max-KL-stable). Thus, by composition and postprocessing, $\monitor$ is $(2\eps,\delta)$-max-KL stable.
We can hence apply Lemma~\ref{lem:MKLCondExp} to obtain
\begin{equation}\label{eq:MKLMonitorCondExp2}
\ex{\samples, \monitor}{\query^*(\sample_{\trial^*}) - \query^*(\dist) \mid (\query^*, \trial^*) = \monitor(\samples)}  \leq 2\left(e^{2\eps} - 1 + \trials\delta\right)\Delta n < 8\eps\Delta n.
\end{equation}

Now we can apply Lemma \ref{lem:EMutility} with $f(\query,\trial) =  \query(\sample_\trial)-\query(\dist)$ and $\eta = \frac{\eps}{\Delta}$ to get 
\begin{equation}
\ex{\query^*,\trial^*}{f(\query^*,\trial^*)} \geq \max_{(\query,\trial)\in F} f(\query,\trial) - \frac{\Delta}{\eps} \log|F| =  \max_{\trial \in [\trials]}  \left|\query_\trial(\sample_\trial)-\query_\trial(\dist)\right| - \frac{\Delta}{\eps} \log (2T). \label{eq:Utility}
\end{equation}

Combining \eqref{eq:MKLMonitorCondExp2} and \eqref{eq:Utility} gives
\begin{equation}\label{eq:MKLMonitorCondExp2Combined}
\ex{\samples, \monitor}{\max_{\trial \in [\trials]}  \left|\query_\trial(\sample_\trial)-\query_\trial(\dist)\right| } - \frac{\Delta}{\eps} \log (2T) 
\leq \ex{\samples, \monitor}{\query^*(\sample_{\trial^*}) - \query^*(\dist) \mid (\query^*, \trial^*) = \monitor(\samples)} < 8\eps\Delta n.
\end{equation}

To complete the proof, we assume, for the sake of contradition, that $\mechanism$ has a high enough probability of outputting a query $\query$ such that  
$\left| \query(\dist) -\query(\sample) \right|$ is large.  
To obtain a contradiction from this assumption, we need the following natural claim (analogous to Claim \ref{clm:amplify}) about the output of the monitor.
\begin{claim}
If $$\pr{\sample,\mechanism}{ \left| \query(\dist) -\query(\sample) \right| \geq 18\eps\Delta n } \geq \frac{\delta}{\eps},$$ then
$$\pr{\samples, \monitor}{\max_{\trial \in [\trials]}  \left|\query_\trial(\sample_\trial)-\query_\trial(\dist)\right| \geq 18\eps\Delta n } \geq 1 - \left( 1 - \frac{\delta}{\eps} \right)^T \geq \frac12.$$
\end{claim}
Thus 
\begin{equation}\label{eq:LargeError}
\ex{\samples, \monitor}{\max_{\trial \in [\trials]}  \left|\query_\trial(\sample_\trial)-\query_\trial(\dist)\right|} \geq 9\eps\Delta n.
\end{equation}
Combining \eqref{eq:MKLMonitorCondExp2Combined} and \eqref{eq:LargeError} gives $$ 9\eps\Delta n - \frac{\Delta}{\eps} \log (2T) \leq 8 \eps \Delta n,$$ which simiplies to $$\log(2\eps/\delta) \geq \log(2T) \geq \eps^2 n .$$
This contradicts the assumption that $n\geq\frac{1}{\eps^2}\log(\frac{4\eps}{\delta})$ and hence completes the proof.
\end{proof}

\subsection{Optimality}

We now show that our connection between max-KL stability and generalization (Theorem~\ref{thm:generalization} and Theorem~\ref{thm:MKLtoAdaptive}) is optimal.

\begin{lemma}\label{lem:lower}
Let $\alpha>\delta>0$, let $n\geq\frac{1}{\alpha}$, and let $\Delta\in[0,1]$.
Let $\cU$ be the uniform distribution over $[0,1]$.
There exists a $(0,\delta)$-max-KL stable algorithm $\cA \from [0,1]^n \to \querysetDelta$ such that if $\samples\getsr\cU^n$ and if $\query\getsr\cA(\samples)$ then
$$
\Pr[\query(\samples) - \query(\cU) \geq \alpha\Delta n]\geq\frac{\delta}{2\alpha}.
$$
\end{lemma}

\begin{proof}

Consider the following simple algorithm, denoted as $\cB$. On input a database $\sample$, output $\sample$ with probability $\delta$, and otherwise output the empty database. Clearly, $\cB$ is $(0,\delta)$-max-KL stable. Now construct the following algorithm $\cA$.

\begin{figure}[h]
\begin{framed}
\begin{algorithmic}
\INDSTATE[0]{\textbf{Input:} 
A database $\samples\in[0,1]^n$. We think of $\samples$ as $\frac{1}{\alpha}$ databases of size $\alpha n$ each: $\samples = (\sample_1,\dots,\sample_{1/\alpha})$.}
\INDSTATE[1]{For $1\leq i\leq1/\alpha$ let $\hat{\sample}_i=\cB(\sample_i)$.}
\INDSTATE[1]{Let $p:[0,1]\rightarrow\{0,1\}$ where $p(x)=1$ iff $\exists i$ s.t.\ $x\in\hat{\sample}_i$.}
\INDSTATE[1]{Define $\query_p:[0,1]^n\rightarrow\R$ where $\query_p(\sample)=\Delta\sum_{x\in\sample}p(x)$ (note that $\query_p$ is a $\Delta$-sensitive query, and that it is a statistical query if $\Delta=1/n$).}
\INDSTATE[0]{\textbf{Output:} $\query_p$.}
\end{algorithmic}
\end{framed}
%\vspace{-5mm}
%\caption{Foo}
\end{figure}

As $\cB$ is $(0,\delta)$-max-KL stable, and as $\cA$ only applies $\cB$ on disjoint databases, we get that $\cA$ is also $(0,\delta)$-max-KL stable.

Suppose $\samples = (\sample_1,\dots,\sample_{1/\alpha})$ contains i.i.d.\ samples from $\cU$, and consider the execution of $\cA$ on $\samples$.
Observe that the predicate $p$ evaluates to 1 only on a finite number of points from $[0,1]$, and hence, we have that $\query_p(\cU)=0$.
Next note that $\query_p(\samples)=\alpha\Delta n\cdot|\{ i : \hat{\sample}_i=\sample_i \}|$. Therefore, if there exists an $i$ s.t.\ $\hat{\sample}_i=\sample_i$ then $\query_p(\samples)- \query_p(\cU)\geq\alpha\Delta n$.
The probability that this is not the case is at most
$$
(1-\delta)^{1/\alpha}\leq e^{-\delta/\alpha} \leq 1-\frac{\delta}{2\alpha},
$$
ans thus, with probability at least $\frac{\delta}{2\alpha}$, algorithm $\cA$ outputs a $\Delta$-sensitive query $\query_p$ s.t.\ $\query_p(\samples)-\query_p(\cU)\geq\alpha\Delta n$.
\end{proof}

In particular, using Lemma~\ref{lem:lower} with $\alpha=\varepsilon$ shows that the parameters in Theorem~\ref{thm:generalization} are tight.

\ifnum\short = 0
\addcontentsline{toc}{section}{Acknowledgements}
\subsubsection*{Acknowledgements}
We thank Mark Bun, Moritz Hardt, Aaron Roth, and Salil Vadhan for many helpful discussions.
\fi

\ifnum\short=1
\vfill\eject \small
\bibliographystyle{alpha}
\bibliography{references}
\else
\addcontentsline{toc}{section}{References}
\bibliographystyle{alpha}
\bibliography{references}
\fi

\end{document}